\documentclass{article}

\usepackage[preprint]{nips_2018}

\usepackage[utf8]{inputenc} 
\usepackage[T1]{fontenc}    
\usepackage{hyperref}       
\usepackage{url}            
\usepackage{booktabs}       
\usepackage{amsfonts}       
\usepackage{nicefrac}       
\usepackage{microtype}      
\usepackage{times}
\usepackage{graphicx}
\usepackage{natbib}
\usepackage{algorithm}
\usepackage{algorithmic}
\usepackage{hyperref}
\usepackage{amsfonts}
\usepackage{amsmath}
\usepackage[psamsfonts]{amssymb}
\usepackage{amstext}
\usepackage{amsthm}
\usepackage{latexsym}
\usepackage{graphicx}
\usepackage{subfig}

\usepackage{enumerate}
\usepackage{algorithm}
\usepackage{algorithmic}
\usepackage{url}
\usepackage{dsfont}
\usepackage[mathscr]{euscript}
\usepackage{booktabs}
\usepackage{makecell}

\usepackage{tabularx}
\usepackage{xcolor, colortbl}

\definecolor{Gray}{gray}{0.85}
\newcolumntype{g}{>{\columncolor{Gray}}c}

\makeatletter
\newtheorem*{rep@theorem}{\rep@title}
\newcommand{\newreptheorem}[2]{%
\newenvironment{rep#1}[1]{%
 \def\rep@title{#2 \ref{##1}}%
 \begin{rep@theorem}}%
 {\end{rep@theorem}}}
\makeatother

\hypersetup{
  colorlinks   = true, 
  urlcolor     = blue, 
  linkcolor    = blue, 
  citecolor   = blue 
}

\def\Rset{\mathbb{R}}

\DeclareMathOperator*{\E}{\mathbb{E}}

\DeclareMathOperator*{\argmin}{\rm argmin}

\newcommand{\set}[1]{\{#1\}}

\newcommand{\h}{\widehat}

\newcommand{\cD}{\mathcal{D}}
\newcommand{\cL}{\mathcal{L}}

\newcommand{\cX}{\mathcal{X}}
\newcommand{\cY}{\mathcal{Y}}

\newcommand{\e}{\epsilon}

\newcommand{\wt}{\widetilde}

\newcommand{\sD}{{\mathscr D}}

\newcommand{\sfD}{{\mathsf D}}
\newcommand{\sfd}{{\mathsf d}}
\newcommand{\msD}{{\mathscr D}^{1}} 
\newcommand{\msU}{{\mathscr U}^{1}} 
\newcommand{\scrD}{{\mathscr D}}
\newcommand{\scrU}{{\mathscr U}}

\newtheorem{theorem}{Theorem}
\newreptheorem{theorem}{Theorem}
\newtheorem{lemma}[theorem]{Lemma}
\newreptheorem{lemma}{Lemma}

\newtheorem{corollary}[theorem]{Corollary}
\newreptheorem{corollary}{Corollary}
\newtheorem{proposition}[theorem]{Proposition}
\newreptheorem{proposition}{Proposition}
\newcommand{\ignore}[1]{}

\newcolumntype{P}[1]{>{\centering\arraybackslash}p{#1}}

\title{Algorithms and Theory \\ for Multiple-Source Adaptation}
\author{
Judy Hoffman \\
  CS Department Stanford University\\
  Stanford,  CA 94305 \\
  \texttt{jhoffman@cs.stanford.edu} \\
   \And
  Mehryar Mohri \\
  Courant Institute and Google \\
  New York, NY 10012 \\
  \texttt{mohri@cims.nyu.edu} \\
   \AND
   Ningshan Zhang \\
   New York University \\
   New York, NY 10012 \\
   \texttt{nzhang@stern.nyu.edu} \\
}

\begin{document}

\maketitle

\begin{abstract}
This work includes a number of novel contributions for the multiple-source adaptation problem. We present new normalized solutions with strong theoretical guarantees for the cross-entropy loss and other similar losses. We also provide new guarantees that hold in the case where the conditional probabilities for the source domains are distinct. Moreover, we give new algorithms for determining the distribution-weighted combination solution for the cross-entropy loss and other losses. We report the results of a series of experiments with real-world datasets. We find that our algorithm outperforms competing approaches by producing a single robust model that performs well on  any target mixture distribution. Altogether, our theory, algorithms, and empirical results provide a full solution for the multiple-source adaptation problem with very practical benefits.
\end{abstract}

\section{Introduction}
\label{sec:intro}

In many modern applications, often the learner has access to
information about several source domains, including accurate
predictors possibly trained and made available by others, but no
direct information about a target domain for which one wishes to
achieve a good performance. The target domain can typically be viewed
as a combination of the source domains, that is a mixture of their
joint distributions, or it may be close to such mixtures.
In addition, often the learner does not have access to all source data
simultaneously, for legitimate reasons such as privacy, storage limitation, etc. 
Thus the learner cannot simply pool all source data together to learn a predictor.

Such problems arise commonly in speech recognition where different
groups of speakers (domains) yield different acoustic models and the
problem is to derive an accurate acoustic model for a broader
population that may be viewed as a mixture of the source groups
\citep{liao_icassp13}.  In object recognition, multiple image
databases exist, each with its own bias and labeled categories
\citep{efros_cvpr11}, but the target application may contain images
which most closely resemble only a subset of the available training
data.  Finally, in sentiment analysis, accurate predictors may be
available for sub-domains such as TVs, laptops and CD players, each
previously trained on labeled data, but no labeled data or predictor
may be at the learner's disposal for the more general category of
electronics, which can be modeled as a mixture of the sub-domains
\citep{blitzer_acl07, dredze_nips08}.

The problem of transfer from a single source to a known target domain,
either through unsupervised adaptation techniques \citep{gong_cvpr12,
  long_icml15, ganin_icml15, tzeng_iccv15}, or via lightly supervised
ones (some amount of labeled data from the target domain)
\citep{saenko_eccv10, yang_acmm07, hoffman_iclr13, rcnn}, has been
extensively investigated in the past.  Here, we focus on the problem of
multiple-source domain adaptation and ask how the learner can combine
relatively accurate predictors available for each source domain to
derive an accurate predictor for \emph{any} new mixture target domain?
This is known as the \emph{multiple-source adaption (MSA) problem}   first
formalized and analyzed theoretically by
\cite{MansourMohriRostamizadeh2008,MansourMohriRostamizadeh2009} and
later studied for various applications such as object recognition
\citep{hoffman_eccv12, gong_iccv13, gong_nips13}.
Recently, \cite{zhang2015multi} studied a causal formulation of this problem
and analyzed the same combination rules of \cite{MansourMohriRostamizadeh2008,MansourMohriRostamizadeh2009} 
for classification scenario.
A closely related problem is also that of domain generalization
\citep{pan_tkda2010,MuandetBalduzziScholkopf2013,xu_eccv14}, where
knowledge from an arbitrary number of related domains is combined to
perform well on a previously unseen domain.

\cite{MansourMohriRostamizadeh2008,MansourMohriRostamizadeh2009}  gave strong theoretical guarantees for a distribution-weighted combination for the MSA problem, but they did not provide any algorithmic solution. Furthermore, the solution they proposed could not be used for loss functions such as cross-entropy, which require a normalized predictor. Their work also assumed a deterministic scenario (non-stochastic) with the same labeling function for all source domains.

This work makes a number of novel contributions to the MSA problem.  We give new normalized solutions with strong theoretical guarantees for the cross-entropy loss and other similar losses. Our  guarantees hold even when the conditional probabilities for the source domains are distinct. A by-product of our analysis is  the extension of the theoretical results of \cite{MansourMohriRostamizadeh2008,MansourMohriRostamizadeh2009} to the stochastic scenario, where there is a joint distribution over the input and output space.

Moreover, we give new algorithms for determining the distribution-weighted combination solution for the cross-entropy loss and other losses. We prove that the problem of determining that solution can be cast as a DC-programming (difference of convex) and prove explicit DC-decompositions for the cross-entropy loss and other losses. We also give a series of experimental results with several datasets demonstrating that our distribution-weighted combination solution is remarkably robust.  Our algorithm outperforms competing
approaches and performs well on any target mixture distribution.

Altogether, our theory, algorithms, and empirical results provide a full solution for the MSA problem with very practical benefits.  

\section{Problem setup}
\label{sec:setup}

Let $\cX$ denote the input space and $\cY$ the output space. We consider a multiple-source domain adaptation (MSA) problem in the general
stochastic scenario where there is a
distribution over the joint input-output space, $\cX \times \cY$.  This is a more
general setup than the deterministic scenario in
\citep{MansourMohriRostamizadeh2008, MansourMohriRostamizadeh2009},
where a target function mapping from $\cX$ to $\cY$ is assumed. This
extension is needed for the analysis of the most common and realistic
learning setups in practice.
 We will
assume that $\cX$ and $\cY$ are discrete, but the predictors we
consider can take real values. Our theory can be straightforwardly
extended to the continuous case with summations replaced by integrals in the proofs. 
We will identify a \emph{domain}
with a distribution over $\cX \times \cY$ and consider the scenario
where the learner has access to a predictor $h_k$, for each domain
$\sD_k$, $k=1, \ldots, p$.  

We consider two types of predictor functions $h_k$, and their
associated loss functions $L$ under the {regression model (R)} and  the {probability model (P)}  respectively,
\begin{center}
    \begin{tabular}{lll}
          $h_k  \colon \cX \to \Rset$ & $L\colon \Rset \times \cY \to \Rset_+$ & \text{(\emph{R})} \\
                     $h_k \colon \cX \times \cY \to [0, 1]$
          & $L\colon [0, 1] \to \Rset_+$ &
          \text{(\emph{P})}
      \end{tabular}
  \end{center}
We abuse the notation and write $L(h, x, y)$ to denote the loss of a
predictor $h$ at point $(x, y)$, that is  $L(h(x), y)$ in the
{regression model}, and $L(h(x, y))$ in the {probability model}.
We will denote by
$\cL(\sD, h)$ the expected loss of a predictor $h$ with respect to the
distribution $\sD$:
\begin{align*}
    \cL(\sD, h) &=\mspace{-5mu}\E_{(x, y) \sim \sD}\mspace{-5mu} \big[L( h,x,y) \big]\nonumber 
    = \mspace{-15mu} \sum_{(x, y) \in \cX \times \cY} \mspace{-15mu}\sD(x, y) L(h,x, y).
\end{align*}
Much of our theory only assumes that $L$ is convex and continuous.
But, we will be particularly interested in the case where  in the regression model, $L(h(x), y) = (h(x) - y)^2$ is the squared loss,
and where in the probability model, $L(h(x,y)) = -\log h(x,y)$ is
the cross-entropy loss ($\log$-loss).

We will assume that each $h_k$ is a relatively accurate predictor for
the distribution $\sD_k$:  there exists $\e > 0$ such that
$\cL(\sD_k, h_k) \leq \e$ for all $k \in [p]$.
We will also assume that the loss of the source hypotheses $h_k$ is
bounded, that is $L(h_k,x , y) \leq M$ for all $(x, y) \in \cX \times \cY$
and all $k\in [p]$. 

In the MSA problem, the learner's objective is to combine these predictors to design a predictor with small expected loss on a target domain that could be an arbitrary and unknown mixture of the source domains, the case we are particularly interested in, or even some other arbitrary distribution. It is worth emphasizing that the leaner has no knowledge of the target domain. 

How do we combine the $h_k$s? Can we use a convex combination rule, $\sum_{k=1}^p \lambda_k h_k$, for some $\lambda\in\Delta$? In Appendix~\ref{app:lower_bounds} (Lemmas~\ref{lemma:lower_reg} and \ref{lemma:lower_log}) we show that \emph{no} convex combination rule will perform well even in very simple MSA problems. 
These results generalize a previous lower bound of  \cite{MansourMohriRostamizadeh2008}.  Next, we show that the distribution-weighted combination rule is the right solution.

Extending the definition given by \cite{MansourMohriRostamizadeh2008},
we define the distribution-weighted combination of the functions
$h_k$, $k \in [p]$ as follows. For any $z \in \Delta$, $\eta > 0$,
and $(x,y) \in \cX\times\cY$,
\begin{align}
h_z^\eta(x) 
&= \sum_{k = 1}^p \frac{z_k \msD_k(x) + \eta  \frac{\msU(x)}{p}}{\sum_{k = 1}^p z_k \msD_k(x) + \eta \msU(x)} h_k(x),
  &\text{(\emph{R})} \label{eq:h_reg}\\
 h_z^\eta(x, y) &= \sum_{k = 1}^p\frac{z_k \scrD_k(x, y) + \eta \, \frac{\scrU(x, y)}{p}}{
    \sum_{j = 1}^p z_j \scrD_j(x, y) + \eta \, \scrU(x, y)} h_k(x, y),
 &\text{(\emph{P})} \label{eq:h_prob}
\end{align}
where we denote by $\msD(x)$ the marginal distribution over $\cX$:
$\msD(x) = \sum_{y \in \cY} \sD(x, y)$, and $\msU(x)$ the uniform distribution over $\cX$.
 This extension may seem technically
straightforward in hindsight, but the form of the predictor was not
immediately clear in the stochastic case.

\section{Theoretical analysis}
\label{sec:theory_analysis}

In this section, we present theoretical analyses of the general
multiple-source adaptation setting.  We first introduce our main result for the general stochastic scenario. Next, for the probability model with cross-entropy loss, we introduce a \emph{normalized} distribution weighted combination and prove that it benefits from strong theoretical guarantees.

Our theoretical results rely on the measure of divergence between
distributions.
The one that naturally comes up in our analysis is the \emph{R\'enyi Divergence} \citep{Renyi1961}. 
We will denote by $  \sfd_\alpha(\sD \parallel \sD') = e^{\sfD_\alpha(\sD \parallel \sD')} 
$ the exponential of the $\alpha$-R\'enyi Divergence of two distributions $\sD$ and $\sD'$.
\ignore{
Given a class of distributions $\cD$, we denote by
$\sfd_\alpha(\sD \parallel {\cD})$ the infimum
$\inf_{\sD'\in {\cD}} \sfd_\alpha(\sD \parallel \sD')$.  We will
concentrate on the case where $\cD$ is the class of all mixture
distributions over a set of $k$ source distributions, i.e.,
${\cD} = \set{\sum_{k = 1}^p
  \lambda_k \sD_k \colon \lambda \in \Delta}$, and the case where 
  $\cD^1 = \set{\sum_{k=1}^p \lambda_k \msD_k: \lambda\in \Delta}$.
  }
 More details of the R\'enyi Divergence are given in Appendix~\ref{app:renyi}.

\subsection{Stochastic scenario}
Let $\sD_T$ be an unknown target distribution. 
We will denote by $\sD_T(\cdot |x)$ and $\sD_k(\cdot|x)$ the conditional probability distribution on the target and the source domain respectively. Given the same input $x$, $\sD_T(\cdot |x), \sD_k(\cdot|x),k\in [p]$ are not necessarily the
same. This is a novel extension that was not discussed in \citep{MansourMohriRostamizadeh2009}, where in the deterministic scenario, exactly the same labeling function $f$ is assumed for all source domains. 

\ignore{We introduce some additional notation. 
 Let $\sD_{k,T}(x,y)=\msD_k(x)\sD_T(y|x)$,
and let $\sD_{P,T}$ be the class of mixtures of $\sD_{k,T}$:
$\sD_{P,T} = \big\{ \sum_{k = 1}^p \lambda_k \sD_{k,T}, \lambda \in
  \Delta \big\}$. }
  
For some choice of $\alpha>1$, define $\e_{T}$ by
\begin{equation*}
    \e_{T}= \max_{k\in [p]} \left[\underset{\msD_k(x)}{\mathbb{E}} \sfd_{\alpha}\left(
\sD_T(\cdot | x) \hspace{-0.05cm}\parallel \hspace{-0.05cm}\sD_k(\cdot | x)\right)^{\alpha-1}\right]^{\frac{1}{\alpha}}
\hspace{-0.1cm}\e^{\frac{\alpha-1}{\alpha}}M^{\frac{1}{\alpha}}.
\end{equation*}
When the average divergence is small, $\alpha$ can be chosen to be very large and $\e_T$ is close to $\e$. 

Let $\sD_T$ be a mixture of source distributions, such that  $\msD_T \in  \cD^1 =\set{\sum_{k=1}^p \lambda_k \msD_k: \lambda\in \Delta}$ in the regression model (R), or $\sD_T \in \cD=\set{\sum_{k = 1}^p
  \lambda_k \sD_k \colon \lambda \in \Delta}$ in the probability model (P).  

\begin{theorem}
\label{th:distinct_mixture}
For any $\delta >0$, 
there exists $\eta >0$ and $z \in \Delta$ such that the following inequalities
hold for any $\alpha>1$:
\begin{align*}
\cL(\sD_T, h_z^\eta) &\leq \e_{T} + \delta,&(R) \\
\cL(\sD_T, h_z^\eta)  &\leq \e + \delta. &(P)
\end{align*}
\end{theorem}
\ignore{
\begin{theorem}
\label{th:distinct}
For any $\delta >0$, 
there exists $\eta >0$ and $z \in \Delta$ such that the following inequalities
hold for any $\alpha>1$:
\begin{align*}
\cL(\sD_T, h_z^\eta) &\leq \Big[ (\e_{T} + \delta) \sfd_\alpha (\sD_T \parallel \sD_{P,T})\Big]^
{\frac{\alpha-1}{\alpha}} M^{\frac{1}{\alpha}}, &(R) \\
\cL(\sD_T, h_z^\eta)  &\leq \Big[(\e + \delta) \, \sfd_\alpha(\sD_T \parallel \cD)
\Big]^{\frac{\alpha - 1}{\alpha}} M^{\frac{1}{\alpha }}. &(P)
\end{align*}
\end{theorem}
}
The proof is given in Appendix~\ref{app:theory}. The learning
guarantees for the regression and the probability model are slightly different, since the definitions of the distribution-weighted combinations are different for the two models. 
Theorem~\ref{th:distinct_mixture} shows the existence of $\eta> 0$ and a
mixture weight $z\in\Delta$ with a remarkable property: in the regression model (R),  for any target
distribution $\sD_T$ whose conditional probability $\sD_T(\cdot|x)$ is on average not too far away from  $\sD_k(\cdot|x)$ for any $k\in[p]$, 
and $\msD_T \in \cD^1$, the loss of $h_z^\eta$ on $\sD_T$ is  small. It is even more remarkable that, in the probability model (P), the loss of $h_z^\eta$  is at most $\e$ on any target distribution $\sD_T \in \cD$.
 Therefore, $h_z^\eta$ is a robust hypothesis with favorable property for any such target distribution $\sD_T$.  

In many learning tasks, it is reasonable to assume that the conditional
probability of the output labels is the same for all source domains.
For example, a dog picture represents 
a dog regardless of whether the dog appears in an individual's personal set of pictures or in a broader 
database of pictures from multiple individuals. 
This is a straightforward extension of the assumption adopted by 
\citet{MansourMohriRostamizadeh2008} in the deterministic scenario, where exactly the same labeling function $f$ is assumed for all source domains.  Then $\sD_T(\cdot|x)=\sD_k(\cdot | x),\forall k\in[p]$. 
By definition, 
 $\sfd_{\alpha}\left(
\sD_T(\cdot | x) \hspace{-0.08cm}\parallel \hspace{-0.08cm}\sD_k(\cdot | x)\right)=1$. Let $\alpha \to +\infty$, we recover the main result  of \cite{MansourMohriRostamizadeh2008}.  

\begin{corollary}
\label{th:mixture}
Assume the conditional probability $\sD_k(\cdot|x)$ does not depend on $k$. Let $\sD_\lambda$ be an arbitrary mixture of source domains,   $\lambda \in \Delta$. For any $\delta > 0$, there exists $\eta > 0$ and $z\in \Delta$, such
that $\cL(\sD_\lambda,h_z^\eta)\leq \e +\delta$.
\end{corollary}
 Corollary~\ref{th:mixture} shows the existence of a mixture weight $z \in \Delta$ and $\eta > 0$ with a remarkable property: for any $\delta > 0$,
regardless of which mixture weight $\lambda \in \Delta$ defines the
target distribution, the loss of $h_z^\eta$ is at most $\e + \delta$,
that is arbitrarily close to $\e$. $h_z^\eta$ is therefore a
\emph{robust} hypothesis with a favorable property for any mixture
target distribution.

To cover the realistic cases in applications,
we further extend this result to the case  where the distributions $\sD_k$ are not directly available to the learner, and instead estimates $\h\sD_k$ have been derived from data, and further to the case where the target distribution $\sD_T$ is not a mixture of source distributions (Corollary~\ref{th:estimate} in Appendix~\ref{app:theory}).  We will denote by $\h h_z^\eta$ the distribution-weighted combination rule based on the estimates $\h\sD_k$. Our learning guarantee for $\h h_z^\eta$ depends on the R\'enyi divergence of $\h \sD_k$ and $\sD_k$, as well as the R\'enyi divergence of $\sD_T$ and the family of source mixtures. 

\subsection{Probability model with the cross-entropy loss}
Next, we discuss the special case where $L$ coincide with
the cross-entropy loss in the probability model, and present an analysis for a normalized distribution-weighted combination solution. 
This analysis is a complement to Theorem~\ref{th:distinct_mixture}, which only works for the unnormalized hypothesis $h_z^\eta(x,y)$.

The cross-entropy loss assumes normalized hypotheses. Thus, the source functions are normalized for every $x$:
$    \sum_{y\in \cY} h_k(x,y) = 1,\ \forall x\in \cX, \forall k\in [p]$. 
For any $z \in \Delta$,
$\eta > 0$, we define the normalized
weighted combination $\overline h_z^\eta(x,y)$ that is based on  
$h_z^\eta(x,y)$ in~\eqref{eq:h_prob}:
\begin{equation*}
    \label{eq:normalized_h}
    \overline h_z^\eta(x,y) = \frac{h_z^\eta(x,y)}{\overline h_z^\eta(x)},
    \quad \text{where } \overline  h_z^\eta(x)=\sum_{y\in\cY}h_z^\eta(x,y).
\end{equation*}

We will first assume the conditional probability $\sD_k(\cdot|x)$ does not depend on $k$.  
\begin{theorem}
\label{th:normalized}
Assume there exists $\mu>0$ such that $\scrD_k(x,y) \ge \mu \scrU(x,y)$ 
for all $k\in[p]$ and $(x,y)\in \cX \times \cY$. Then, for any $\delta > 0$, there exists
$\eta > 0$ and $z\in \Delta$, such that $\cL(\scrD_\lambda,\overline h_z^\eta)\leq
\e +\delta$ \ for any mixture parameter $\lambda \in \Delta$.
\end{theorem}
The result of Theorem~\ref{th:normalized} admits the same favorable property as that of Corollary~\ref{th:mixture}. It can also be extended to the case of arbitrary target distributions and estimated densities.
When the conditional probabilities are different across the source domains, we propose a marginal distribution-weighted combination rule, which is already normalized. We can directly apply Theorem~\ref{th:distinct_mixture} to it and achieve favorable guarantees.  More details are given in Appendix~\ref{app:log-loss}.

These results are non-trivial and important, as they provide a guarantee for an accurate and robust predictor for a commonly used loss function, the cross-entropy loss. 

\section{Algorithms}
\label{sec:algorithm}
We have shown that, for both the {regression} and
the {probability model}, there exists a vector $z$ defining a
distribution-weighted combination hypothesis $h_z^\eta$ that admits
very favorable guarantees. But how we find a such $z$?  This is a key question in the MSA problem which was not addressed by
\cite{MansourMohriRostamizadeh2008,MansourMohriRostamizadeh2009}: no
algorithm was previously reported to determine the mixture parameter
$z$ (even for the deterministic scenario). Here, we give an algorithm for determining that vector $z$.

\ignore{Our proofs use Brouwer's Fixed Point Theorem to show the existence of $z \in \Delta$. To find the fixed point,
one general approach consists of using the combinatorial algorithm of
\cite{Scarf1967} based on simplicial partitions, which makes use of
a result similar to Sperner's Lemma \citep{Kuhn1968}. However, that
algorithm is costly in practice and its computational complexity is
exponential in $p$.\ignore{  Other algorithms were later given by
\cite{Eaves1972} and \cite{Merrill1972}. But, it has been shown more
generally by }
More generally, any (general) algorithm for computing Brouwer's Fixed
Point based on function evaluations must make a
number of evaluation calls that is exponential both in $p$ and the
number of digits of approximation accuracy in the worst case 
\citep{HirschPapadimitriouVavasis1989,ChenDeng2008}.}


In this section, we give practical and efficient algorithms for
finding the vector $z$ in the important cases of the squared loss in
the {regression model}, or the cross-entropy loss in the {probability model}, by
leveraging the differentiability of the loss functions.  We first show
that $z$ is the solution of a general optimization problem. Next, we
give a DC-decomposition (difference of convex decomposition) of the
objective for both models, thereby proving an explicit DC-programming
formulation of the problem.  This leads to an efficient DC algorithm
that is guaranteed to converge to a stationary point.  Additionally,
we show that it is straightforward to test if the solution obtained is
the global optimum. While we are not proving that the local stationary
point found by our algorithm is the global optimum, empirically, we
observe that that is indeed the case.\ignore{ Note that the global
  minimum of our DC-programming problem can also be found using a
  cutting plane method of \cite{HorstThoai1999} that does not admit
  known algorithmic convergence guarantees or a branch-and-bound
  algorithm with exponential convergence \citep{HorstThoai1999}.}

\subsection{Optimization problem}

Theorem~\ref{th:distinct_mixture} shows that the hypothesis $h_z^\eta$ based
on the mixture parameter $z$ benefits
from a strong generalization guarantee. 
A key step in proving Theorem~\ref{th:distinct_mixture}
is to show
the following lemma.

\begin{lemma}
\label{lemma:brouwer} 
For any $\eta, \eta' > 0$, there exists $z \in \Delta$, with
$z_k \neq 0$ for all $k \in [p]$, such that the following holds for
the distribution-weighted combining rule $h_z^\eta$:
\begin{equation}
\label{eq:property}
\forall k \in [p], \quad \cL(\sD_k, h_z^{\eta}) \leq \sum_{j = 1}^p z_j \cL(\sD_j, h_z^\eta)  + \eta'.
\end{equation}
\end{lemma}
Lemma~\ref{lemma:brouwer} indicates that for the solution $z$, $h_z^\eta$ has essentially the same loss on all source domains.
 Thus, our problem consists of finding a parameter $z$ verifying this property.  This, in turn, can be
formulated as a min-max problem:
$\min_{z \in \Delta} \max_{k \in [p]} \cL(\sD_k, h_z^\eta) - \cL(\sD_z, h_z^\eta),$
which can be equivalently formulated as the following optimization problem:
\begin{equation}
\label{eq:opt}
\min_{z \in \Delta, \gamma \in \Rset}  \ \gamma\quad
\text{s.t.}  \ \cL(\sD_k, h_z^\eta) - \cL(\sD_z, h_z^\eta) \leq
\gamma, \, \forall k \in [p].
\end{equation}

\subsection{DC-decomposition}
\label{sec:dc}

We provide explicit DC decompositions of the objective of
Problem~\eqref{eq:opt} for the {regression model} with the squared
loss and the {probability model} with the cross-entropy loss. The  full derivations are given in Appendix~\ref{app:dcd}.

We first rewrite $h^\eta_z$ as the division of two affine functions for the {regression model (R)}
and the {probability model (P)}:
$h_z = J_z/K_z$, where
\ignore{\begin{tabular}{lll}
    $\displaystyle J_z = \sum_{k = 1}^p z_k \msD_k h_k + \frac{\eta}{p} \msU h_k$,
    & $\displaystyle K_z = \msD_z  + \eta \, \msU$ & \emph{(R)} \\
    $\displaystyle J_z = \sum_{k = 1}^p z_k \scrD_k h_k + \frac{\eta}{p} \scrU h_k$,
    &$\displaystyle K_z = D_z  + \eta  \, \scrU$ & \emph{(P)}
\end{tabular}}

\begin{tabular}{lll}
    $\displaystyle J_z(x) = \sum_{k = 1}^p z_k \msD_k(x) h_k(x) + \frac{\eta}{p} \msU(x)h_k(x)$,
    & $\displaystyle K_z(x) = \msD_z(x)  + \eta \, \msU(x)$, & \emph{(R)} \\
    $\displaystyle J_z(x, y) = \sum_{k = 1}^p z_k \scrD_k(x, y) h_k(x, y) + \frac{\eta}{p} \scrU(x, y) h_k(x,y)$,
    &$\displaystyle K_z(x, y) = D_z(x, y)  + \eta  \, \scrU(x, y)$ & \emph{(P)}
\end{tabular}

\begin{proposition}[Regression model, square loss]
\label{prop:dc-regression}
Let $L$ be the squared loss. Then, for any $k \in [p]$,
$\cL(\sD_k, h_z^\eta) - \cL(\sD_z, h_z^\eta) = u_k(z) - v_k(z)$, where
$u_k$ and $v_k$ are convex functions defined for all $z$ by
\begin{align*}
u_k(z) &= \cL\left(\sD_k+\eta\msU \sD_k(\cdot | x),h_z^\eta \right) 
       -2M\sum_{x}(\msD_k+\eta\msU)(x)\log K_z(x),\\
v_k(z) &= \cL\left(\sD_z+\eta\msU \sD_k(\cdot | x),h_z^\eta \right) 
       -2M\sum_{x}(\msD_k+\eta\msU)(x)\log K_z(x).
\end{align*}
\end{proposition}
\ignore{\emph{Proof sketch.} 
First decompose $L(h_z^\eta(x),y) = f_z(x,y) - g_z(x)$,
with $f_z(x,y) = (h_z^\eta(x)-y)^2 - 2M\log K_z(x)$,
and $g_z(x) = -2M\log K_z(x)$. 
$g_z(x)$ is convex since $-\log K_z$ is convex as the composition of the convex function $-\log$ with an affine function.  
$f_z(x,y)$ is convex since its Hessian matrix is positive semi-definite when $L(h_z^\eta(x),y)\le M$.
Plug in the DC-decomposition of $L(h_z^\eta(x),y)$, rearrange terms, and prove the convexity of $J_z^2/K_z$ (by calculating its Hessian), we show that 
$u_k(z)-v_k(z)$ is a DC-decomposition.\qed}

\begin{proposition}[Probability model, cross-entropy loss]
\label{prop:dc-probability}
  Let $L$ be the cross-entropy loss. Then, for $k \in [p]$,
  $\cL(\scrD_k, h_z^\eta) - \cL(\scrD_z, h_z^\eta) = u_k(z) - v_k(z)$,
  where $u_k$ and $v_k$ are convex functions:
\begin{align*}
    u_k(z) & = -\sum_{x, y}  \big[ \scrD_k(x, y) + \eta \, \scrU(x,  y) \big] \log J_z(x, y),\\
    v_k(z) & = \sum_{x, y} K_z(x, y) \log \left[  \frac {K_z(x, y)}{J_z(x, y)} \right]    
           - \left[ \scrD_k(x, y) + \eta \, \scrU(x,  y) \right] \log K_z(x, y).
\end{align*}
\end{proposition}
\vskip -.15in
\ignore{\emph{Proof sketch.} 
$u_k$ is convex since $(-\log J_z)$ is convex as the composition of the
convex function $(-\log)$ with an affine function. Similarly,
$(-\log K_z)$ is convex, thus the second term of $v_k$ is convex. 
The first term of $v_k$ can be
written in terms of the unnormalized relative entropy:
\begin{align*}
     \sum_{x, y} K_z(x, y) \log  \left[ \frac {K_z(x, y)}{J_z(x, y)} \right] 
     = B(K_z \parallel J_z) + \sum_{(x, y)} K_z(x,y) - J_z(x,y).
\end{align*}
The unnormalized relative entropy $B(\cdot \parallel \cdot)$ is
jointly convex \citep{CoverThomas2006}, thus $B(K_z \parallel J_z)$
is convex as the composition of the unnormalized relative entropy with
affine functions (for each of its two arguments).  $(K_z - J_z)$ is an
affine function of $z$ and is therefore convex too.\qed}

\subsection{DC algorithm}

Our DC decompositions prove that the optimization problem~\eqref{eq:opt}
can be cast as the following variational form of a DC-programming
problem \citep{TaoAn1997,TaoAn1998,SriperumbudurLanckriet2009}:
\begin{equation}
\label{eq:dc}
 \min_{z \in \Delta, \gamma \in \Rset}   \gamma \quad
\text{s.t.}    \big( u_k(z) - v_k(z)  \leq
\gamma \big) \wedge 
\big(-z_k  \leq 0 \big)  \wedge
\big(\textstyle \sum_{k = 1}^p z_k - 1 = 0\big), \quad \forall k\in [p].
\end{equation}
The DC-programming algorithm works as follows. Let $(z_t)_t$ be the sequence defined by repeatedly solving
the following convex optimization problem:
\begin{align}
\label{eq:dca}
 z_{t + 1} &\in \argmin_{z, \gamma \in \Rset}  \ \gamma\\\nonumber
\text{s.t.} & \  \big( u_k(z) - v_k(z_t) - (z - z_t) \nabla v_k(z_t)  \leq
\gamma \big) \wedge
\big( -z_k  \leq 0 \big) \wedge
\big( \textstyle \sum_{k = 1}^p z_k - 1  = 0 \big), \quad \forall k\in[p],
\end{align}
where $z_0 \in \Delta$ is an arbitrary starting value. Then, $(z_t)_t$
is guaranteed to converge to a local minimum of Problem~\eqref{eq:opt}
\citep{YuilleRangarajan2003,SriperumbudurLanckriet2009}. Note that
Problem~\eqref{eq:dca} is a relatively simple optimization problem:
$u_k(z)$ is a weighted sum of the negative logarithm of an affine function of $z$, plus a weighted sum of rational functions of $z$ (squared loss),
and all other terms appearing in the constraints are
affine functions of $z$.

Problem~\eqref{eq:opt} seeks a parameter $z$ verifying
$\cL(\sD_k, h_z^\eta) - \cL(\sD_z, h_z^\eta) \leq \gamma$, for all
$k \in [p]$ for an arbitrarily small value of $\gamma$.  Since
$\cL(\sD_z, h_z^\eta) = \sum_{k = 1}^p z_k \cL(\sD_k, h_z^\eta)$ is a
weighted average of the expected losses $\cL(\sD_k, h_z^\eta)$,
$k \in [p]$, the solution $\gamma$ cannot be negative.  Furthermore,
by Lemma~\ref{lemma:brouwer}, a parameter $z$ verifying that
inequality exists for any $\gamma > 0$. Thus, the global solution
$\gamma$ of Problem~\eqref{eq:opt} must be close to zero. This provides
us with a simple criterion for testing the global optimality of the
solution $z$ we obtain using a DC-programming algorithm with a
starting parameter $z_0$.

\section{Experiments}
\label{sec:eval}

This section reports the results of our experiments with our
DC-programming algorithm for finding a robust domain generalization
solution when using squared loss and cross-entropy loss.  
We first evaluate our
algorithm using an artificial dataset assuming known densities where
we may compare our result to the global solution and found that indeed
our global objective approached the known optimum of zero (see
Appendix~\ref{app:more_exp} for more details).  Next, we evaluate our
DC-programming solution applied to real-world datasets: a sentiment
analysis dataset \citep{blitzer_acl07} for squared loss, 
a visual domain adaptation benchmark dataset \emph{Office} \citep{saenko_eccv10},
as well as a generalization of digit recognition task, for cross-entropy loss.

For all real-world datasets, the probability
distributions $\sD_k$ are not readily available to the
learner. However, Corollary~\ref{th:estimate} extends the learning
guarantees of our solution to the case where an estimate $\h \sD_k$ is
used in lieu of the ideal distribution $\sD_k$. Thus, we used standard density estimation 
methods to derive an
estimate $\h\sD_k$ for each $k\in [p]$. 
While density estimation can be a difficult task in general, for our purpose straightforward
techniques are sufficient for our predictor
$\h h_z^\eta$ to achieve a high performance, since
the approximate densities only serve to indicate the relative
importance of each source domain. We give full details of our density estimation procedure in Appendix~\ref{app:more_exp}.

 \begin{figure*}[t]
     \centering
     \subfloat{
         \includegraphics[width=.35\linewidth]{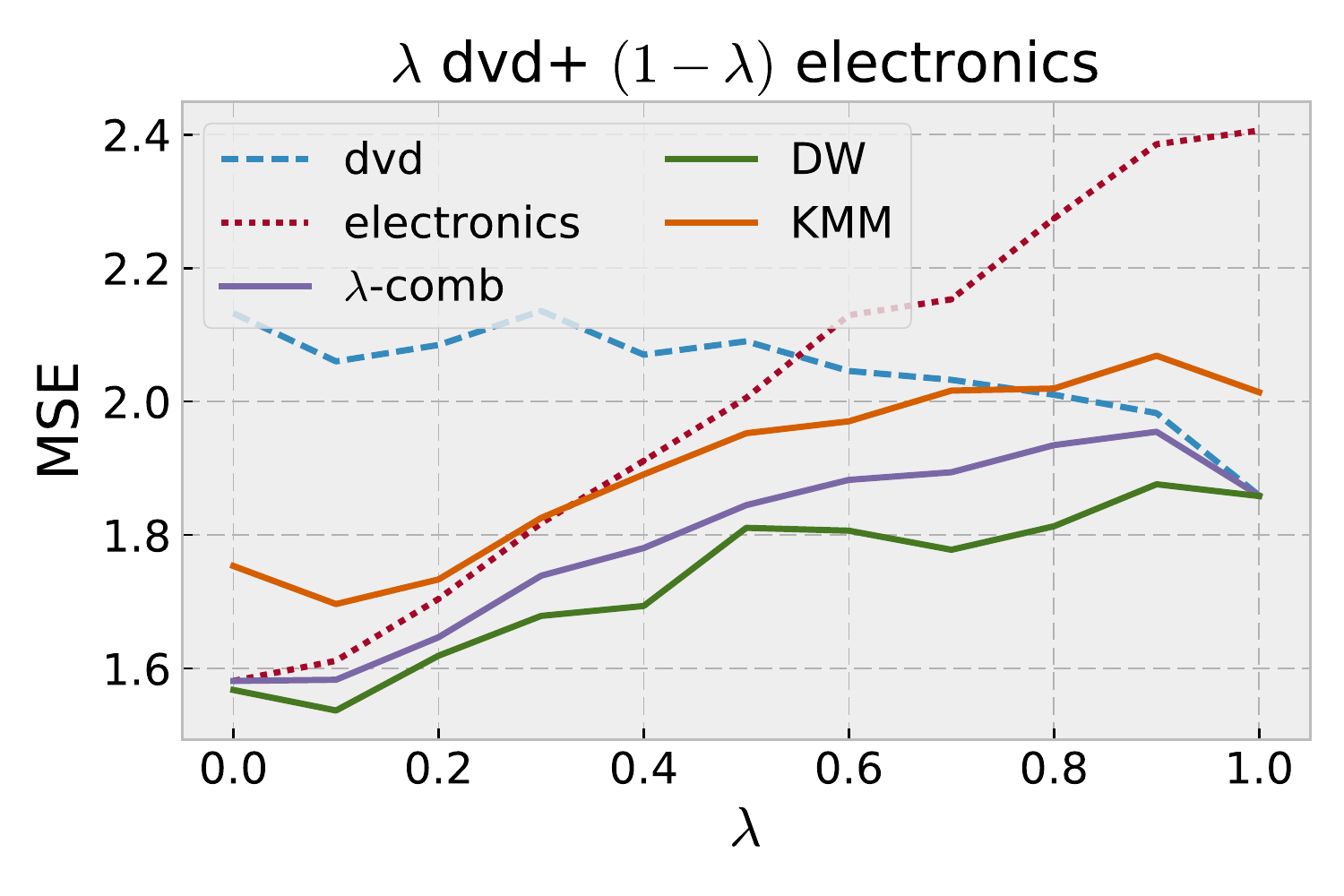}
         \label{fig:2domain_DE}
     }
     \hspace{2cm}%
     \subfloat{
         \includegraphics[width=.35\linewidth]{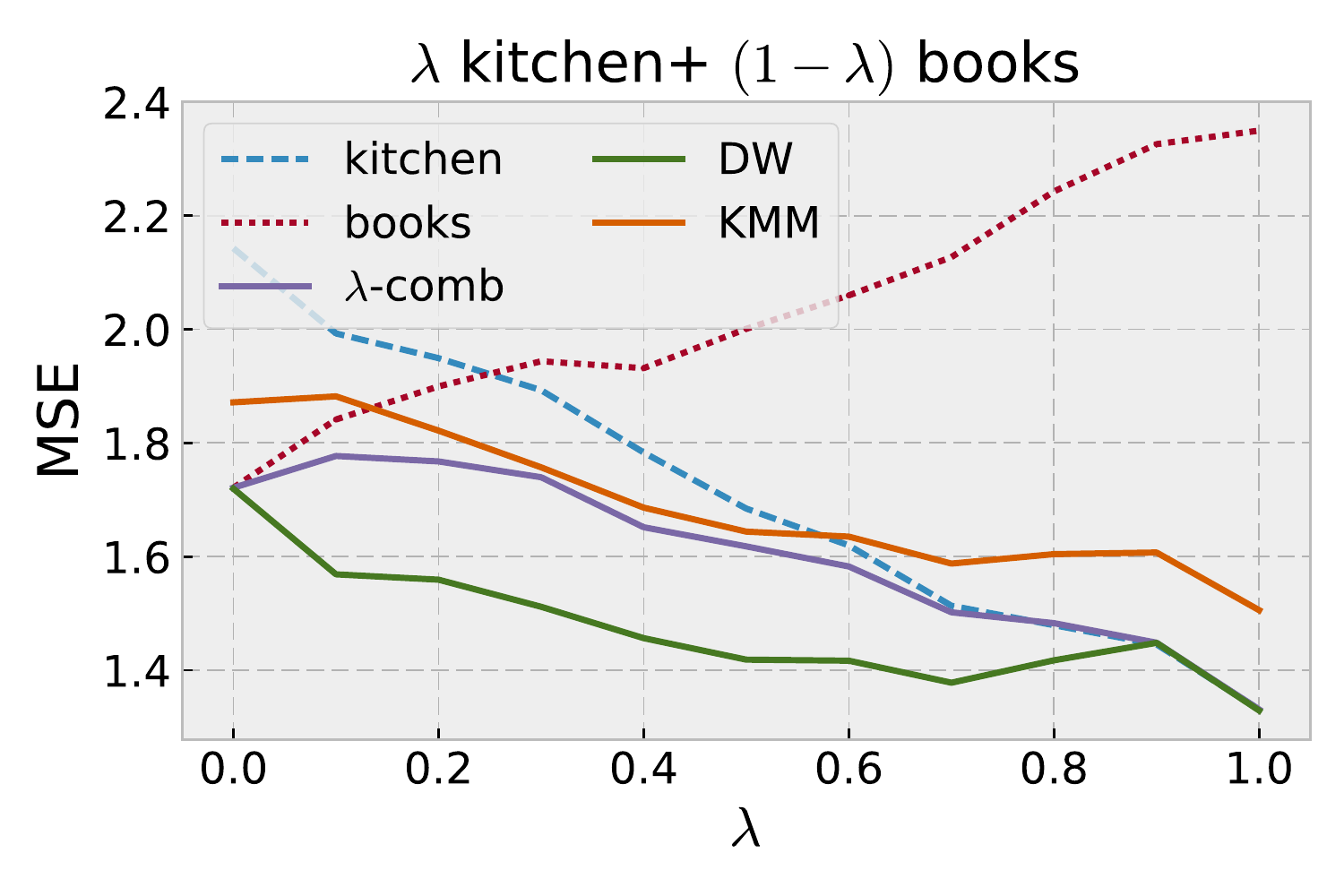}
         \label{fig:2domain_KB}
     }
     \caption{MSE sentiment analysis under mixture of two domains:
     (\emph{left}) dvd and electronics or (\emph{right}) kitchen and books.}
 \end{figure*} 
 
\subsection{Sentiment analysis task for squared loss}
We use the sentiment analysis dataset proposed
by \cite{blitzer_acl07} and used for multiple-source adaptation by
\cite{MansourMohriRostamizadeh2008, MansourMohriRostamizadeh2009}.
This dataset consists of product review text and rating labels taken
from four domains: \texttt{books} (B), \texttt{dvd} (D),
\texttt{electronics} (E), and $\texttt{kitchen}$ (K), with $2000$
samples for each domain. 
%
We defined a vocabulary of $2\mathord,500$ words
that occur at least twice at the intersection of the four
domains. These words were used to define feature vectors, where every
sample is encoded by the number of occurrences of each word. We
trained our base hypotheses using support vector regression
(SVR)
   with same
hyper-parameters as in \citep{MansourMohriRostamizadeh2008,
  MansourMohriRostamizadeh2009}.
 
  \begin{table*}
  \begin{center}
  \caption{MSE on the sentiment analysis dataset of source-only
    baselines for each domain, \texttt{K},\texttt{D},
    \texttt{B},\texttt{E},
    the uniform weighted predictor
    \texttt{unif}, \texttt{KMM}, and the distribution-weighted method
    \texttt{DW} based on the learned $z$. \texttt{DW} outperforms all
    competing baselines. }
  \resizebox{\linewidth}{!}{

  \begin{tabular}{l ccccccccccc}
  \toprule
  & \multicolumn{9}{c}{Test Data}\\
  \cline{2-12}
  & \texttt{K}& \texttt{D} & \texttt{B}	& \texttt{E}	&\texttt{KD}	&\texttt{BE}	& \texttt{DBE} & \texttt{KBE} & \texttt{KDB} &	\texttt{KDB} &	\texttt{KDBE} \\
  \midrule
  \texttt{K}      &       {1.46$\pm$0.08} &       {2.20$\pm$0.14} &       {2.29$\pm$0.13} &       {1.69$\pm$0.12} &       {1.83$\pm$0.08} &       {1.99$\pm$0.10} &       {2.06$\pm$0.07}      &       {1.81$\pm$0.07} &       {1.78$\pm$0.07} &       {1.98$\pm$0.06} &       {1.91$\pm$0.06}\\
  \texttt{D}      &       {2.12$\pm$0.08} &       {1.78$\pm$0.08} &       {2.12$\pm$0.08} &       {2.10$\pm$0.07} &       {1.95$\pm$0.07} &       {2.11$\pm$0.07} &       {2.00$\pm$0.06}      &       {2.11$\pm$0.06} &       {2.00$\pm$0.06} &       {2.01$\pm$0.06} &       {2.03$\pm$0.06}\\
  \texttt{B}      &       {2.18$\pm$0.11} &       {2.01$\pm$0.09} &       {1.73$\pm$0.12} &       {2.24$\pm$0.07} &       {2.10$\pm$0.09} &       {1.99$\pm$0.08} &       {1.99$\pm$0.05}      &       {2.05$\pm$0.06} &       {2.14$\pm$0.06} &       {1.98$\pm$0.06} &       {2.04$\pm$0.05}\\
  \texttt{E}      &       {1.69$\pm$0.09} &       {2.31$\pm$0.12} &       {2.40$\pm$0.11} &       {1.50$\pm$0.06} &       {2.00$\pm$0.09} &       {1.95$\pm$0.07} &       {2.07$\pm$0.06}      &       {1.86$\pm$0.04} &       {1.84$\pm$0.06} &       {2.14$\pm$0.06} &       {1.98$\pm$0.05}\\
  \texttt{unif}   &       {1.62$\pm$0.05} &       {1.84$\pm$0.09} &       {1.86$\pm$0.09} &       {1.62$\pm$0.07} &       {1.73$\pm$0.06} &       {1.74$\pm$0.07} &       {1.77$\pm$0.05}      &       {1.70$\pm$0.05} &       {1.69$\pm$0.04} &       {1.77$\pm$0.04} &       {1.74$\pm$0.04}\\
  \texttt{KMM}    &       1.63$\pm$0.15   &       2.07$\pm$0.12   &       1.93$\pm$0.17   &       1.69$\pm$0.12   &       1.83$\pm$0.07   &       1.82$\pm$0.07   &       1.89$\pm$0.07&       1.75$\pm$0.07   &       1.78$\pm$0.06   &       1.86$\pm$0.09   &       1.82$\pm$0.06\\
  \texttt{DW}(ours)     &       {\bf1.45$\pm$0.08}      &       {\bf1.78$\pm$0.08}      &       {\bf1.72$\pm$0.12}      &       {\bf1.49$\pm$0.06}      &       {\bf1.62$\pm$0.07}      &   {\bf1.61$\pm$0.08}       &       {\bf1.66$\pm$0.05}      &       {\bf1.56$\pm$0.04}      &       {\bf1.58$\pm$0.05}      &       {\bf1.65$\pm$0.04}      &       {\bf1.61$\pm$0.04}\\
  \bottomrule
  \end{tabular}
  }
  	\label{table:sa}
  \end{center}

  \end{table*}

We compare our method (\texttt{DW}) against each
 source hypothesis, $h_k$.  We also compute a privileged baseline
 using the oracle $\lambda$ mixing parameter, \texttt{$\lambda$-comb}:
 $\sum_{k = 1}^p \lambda_k h_k$.  $\lambda$-\texttt{comb} is of course
 not accessible in practice since the target mixture $\lambda$ is not
 known to the user.  We also compare against a previously proposed
 domain adaptation algorithm
 \citep{HuangSmolaGrettonBorgwardtScholkopf2006} known as KMM.  It is important to note that the KMM model requires
 access to the unlabeled target data during adaptation and learns a
 new predictor for every target domain, while \texttt{DW} does not use any target data. Thus KMM operates in a favorable learning setting when compared to our solution.

We first considered the same test scenario as in
 \citep{MansourMohriRostamizadeh2008}, where the target is a mixture
 of two source domains.
 The plots of
 Figures~\ref{fig:2domain_DE} and~\ref{fig:2domain_KB} report the
 results of our experiments. They show that our distribution-weighted
 predictor \texttt{DW} outperforms all baseline predictors despite the
 privileged learning scenarios of \texttt{$\lambda$-comb} and
 \texttt{KMM}. We didn't compare to the ``weighted''
 predictor in empirical studies by \cite{MansourMohriRostamizadeh2008}
 because it is not a real solution, but rather taking the unknown
 target mixture $\lambda$ as $z$ to compute $h_z$.

 Next, we compared the performance of \texttt{DW} with accessible
 baseline predictors on various target mixtures.  Since $\lambda$ is
 not accessible in practice, We replace \texttt{$\lambda$-comb} with
 the uniform combination of all hypotheses (\texttt{unif}),
 $\sum_{k = 1}^p h_k/p$.  
 Table~\ref{table:sa} reports the mean and standard deviations of MSE
 over $10$ repetitions.  Each column corresponds to a different target
 test data source.  Our distribution-weighted method \texttt{DW}
 outperforms all baseline predictors across all test domains. Observe
 that, even when the target is a single source domain, our method
 successfully outperforms the predictor which is trained and tested on
 the same domain.  Results on more target mixtures are available
 in Appendix~\ref{app:more_exp}.

\begin{table}
    \parbox{.4\linewidth}{
        \centering
        \caption{\textbf{Digit} dataset statistics.} 
        \resizebox{\linewidth}{!}{
        \begin{tabular}{lccc}
            \toprule
            &SVHN & MNIST & USPS\\
            & \includegraphics[width=.1\linewidth]{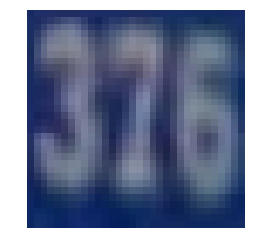}%
            & \includegraphics[width=.1\linewidth]{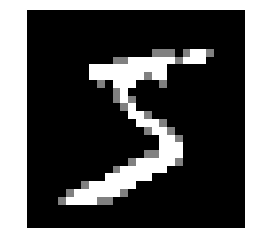}
            & \includegraphics[width=.1\linewidth]{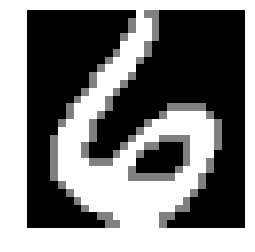}\\
            \midrule
            \# train images & 73,257& 60,000 & 7,291  \\
            \# test images & 26,032& 10,000 & 2,007\\
            image size & 32x32 & 28x28 & 16x16\\
            color & rgb & gray & gray\\
            \bottomrule
        \end{tabular}
    }
    \label{table:digit-stats}
    }
    \hfill
    \parbox{.55\linewidth}{
        \centering
        \caption{\textbf{Digit} dataset accuracy: We report accuracy across six possible test domains.} 
        \resizebox{\linewidth}{!}{%
            \begin{tabular}{l cccccccg}
                \toprule
                & \multicolumn{8}{c}{Test Data}\\
                \cline{2-8}
                & \texttt{svhn}& \texttt{mnist}  & \texttt{usps} & \texttt{mu} & \texttt{su} & \texttt{sm} & \texttt{smu} & mean\\
                \midrule
                CNN-\texttt{s}  &      \bf {92.3} &       {66.9} &       {65.6} &       {66.7} &       {90.4} &       {85.2} &       {84.2} &       {78.8}\\
                CNN-\texttt{m}  &       {15.7} &      \bf {99.2} &       {79.7} &       {96.0} &       {20.3} &       {38.9} &       {41.0} &       {55.8}\\
                CNN-\texttt{u}  &       {16.7} &       {62.3} &      \bf {96.6} &       {68.1} &       {22.5} &       {29.4} &       {32.9} &       {46.9}\\
                CNN-\texttt{unif}       &       {75.7} &       {91.3} &       {92.2} &       {91.4} &       {76.9} &       {80.0} &       {80.7} &       {84.0}\\
                \texttt{DW} (ours)      &       {91.4} &       {98.8} &       {95.6} &       {98.3} &       \bf{91.7} &       \bf{93.5} &       \bf{93.6} &       \bf{94.7}\\
                \midrule
                CNN-\text{joint}        &       {90.9} &       {99.1} &       {96.0} &       \bf {98.6} &       {91.3} &       {93.2} &       {93.3} &       {94.6}\\            
                \bottomrule
            \end{tabular}
        }
        \label{table:digits}
    }
\end{table}

\begin{table*}[h!]
    \caption{\textbf{\emph{Office}} dataset accuracy: We report accuracy across six possible test domains. We show performance all baselines:  CNN-\texttt{a,w,d},  CNN-\texttt{unif},  \texttt{DW} based on the learned $z$, and the jointly trained model CNN-\texttt{joint}.  \texttt{DW} outperforms all competing models.}
\begin{center}
    \resizebox{\linewidth}{!}{
        \begin{tabular}{l cccccccg}
            \toprule
            & \multicolumn{8}{c}{Test Data}\\
            \cline{2-8}

            & \texttt{amazon}& \texttt{webcam} 	& \texttt{dslr}	& \texttt{aw} & \texttt{ad} & \texttt{wd} & \texttt{awd} & mean\\
            \midrule
            CNN-\texttt{a} & 	{\bf 75.7 $\pm$ 0.3} & 	53.8 $\pm$ 0.7 & 	53.4 $\pm$ 1.3 & 	71.4 $\pm$ 0.3 & 	73.5 $\pm$ 0.2 & 	53.6 $\pm$ 0.8 & 	69.9 $\pm$ 0.3 & 	64.5 $\pm$ 0.6 \\
            CNN-\texttt{w} & 	45.3 $\pm$ 0.5 & 	91.1 $\pm$ 0.8 & 	91.7 $\pm$ 1.2 & 	54.4 $\pm$ 0.5 & 	50.0 $\pm$ 0.5 & 	91.3 $\pm$ 0.8 & 	57.5 $\pm$ 0.4 & 	68.8 $\pm$ 0.7 \\
            CNN-\texttt{d} & 	50.4 $\pm$ 0.4 & 	89.6 $\pm$ 0.9 & 	90.9 $\pm$ 0.8 & 	58.3 $\pm$ 0.4 & 	54.6 $\pm$ 0.4 & 	90.0 $\pm$ 0.7 & 	61.0 $\pm$ 0.4 & 	70.7 $\pm$ 0.6 \\
            CNN-\texttt{unif} &	69.7 $\pm$ 0.3 & 	93.1 $\pm$ 0.6 & 	93.2 $\pm$ 0.9 & 	74.4 $\pm$ 0.4 & 	72.1 $\pm$ 0.3 & 	93.1 $\pm$ 0.5 & 	75.9 $\pm$ 0.3 & 	81.6 $\pm$ 0.5 \\
            \texttt{DW} (ours) & 	75.2 $\pm$ 0.4 & 	{\bf 93.7 $\pm$ 0.6} & 	{\bf 94.0 $\pm$ 1.0} & 	{\bf 78.9 $\pm$ 0.4} & 	{\bf 77.2 $\pm$ 0.4} & 	{\bf 93.8 $\pm$ 0.6} & 	{\bf 80.2 $\pm$ 0.3} & 	{\bf 84.7 $\pm$ 0.5} \\
            \midrule
            CNN-\texttt{joint} &	72.1 $\pm$ 0.3 & 	\underline{93.7 $\pm$ 0.5} & 	\underline{93.7 $\pm$ 0.5} & 	76.4 $\pm$ 0.4 & 	76.4 $\pm$ 0.4 & 	93.7 $\pm$ 0.5 & 	79.3 $\pm$ 0.4 & 	83.6 $\pm$ 0.4 \\
            \bottomrule
        \end{tabular}
    }
    \label{table:office}
\end{center}
\end{table*}

\subsection{Recognition tasks for cross-entropy loss}

We consider two real-world domain adaptation tasks:
a generalization of digit recognition task,
and a standard visual adaptation \emph{Office} dataset.

For each individual domain, we train a convolutional neural network (CNN)
and use the output from the softmax score layer as
our base predictors $h_k$.  We compute the uniformly weighted
combination of source predictors, $h_{\texttt{unif}} = \sum_{k=1}^p h_k /p$.
As a privileged baseline, we also train a model on all source data combined, $h_{\texttt{joint}}$.
Note, this approach is often not feasible if independent entities contribute classifiers and densities, but not full training datasets.
Thus this approach is not consistent with our scenario, and it operates in a much more favorable learning setting than our solution.
Finally, our distribution weighted predictor \texttt{DW} is computed with $h_k$s, density estimates, and our learned weighting, $z$.
Our baselines then consists of the classifiers
from $h_k$, $h_\texttt{unif}$, $h_\texttt{joint}$, and \texttt{DW}.

We begin our study with a generalization of digit recognition task, which consists of three digit recognition
datasets: Google Street View House Numbers
(SVHN), MNIST, and
USPS. Dataset statistics as well as example images can be found in
Table~\ref{table:digit-stats}. 
We train the ConvNet (or CNN) architecture following
~\cite{taigman_iclr17} as our source models and joint model. We use the second
fully-connected layer's output as our features for density estimation, 
and the output from the softmax score layer as our predictors. We use the full training
sets per domain to learn the source model and densities. 
Note, these steps are completely isolated from one another and may be performed by
unique entities and in parallel. Finally, for our DC-programming
algorithm we use small subset of 200 real image-label pairs from each domain to learn
the parameter $z$.

Our next experiment uses the standard visual adaptation \emph{Office} dataset,
which has 3 domains: \texttt{amazon}, \texttt{webcam}, and
\texttt{dslr}.  The dataset contains 31 recognition categories of objects commonly
found in an office environment. There are 4110 images total with 2817
from \texttt{amazon}, 795 from \texttt{webcam}, and 498 from
\texttt{dslr}.
We follow the standard protocol from
\cite{saenko_eccv10}, whereby 20 labeled examples are available for
training from the \texttt{amazon} domain and 8 labeled examples are
available from both the \texttt{webcam} and \texttt{dslr} domains. The
remaining examples from each domain are used for testing.
We use the AlexNet~\cite{supervision} ConvNet (CNN) architecture,
and use the output from softmax score layer as our base
predictors, pre-trained on ImageNet and use fc7
activations as our features for density estimation \cite{decaf}.

We report the performance of our method and that of baselines on the
digit recognition dataset in Table~\ref{table:digits}, 
and report the performance on the \emph{Office} dataset in Table~\ref{table:office}.
On both datasets, we evaluate on various test distributions: each individual domain, the combination of
each two domains and the fully combined set.
When the test distribution equals one of the source
distributions, our distribution-weighted classifier successfully
outperforms (\texttt{webcam,dslr}) or maintains performance of
the classifier which is trained and tested on the same domain. 
For the more realistic scenario where the target domain is a mixture of any
two or all three source domains, the performance of our method is
comparable or marginally superior to that of the jointly trained
network, despite the fact that we do not retrain any network
parameters in our method and that we only use a small number of
per-domain examples to learn the distribution weights -- an
optimization which may be solved on a single CPU in a matter of
seconds for this problem. This again demonstrates
the robustness of our distribution-weighted combined classifier to a
varying target domain.

\section{Conclusion}
\label{sec:conclusion}
We presented practically applicable multiple-source domain adaptation
algorithms for the  cross-entropy loss and other similar losses. These algorithms
benefit from very favorable theoretical guarantees that we extended to the stochastic setting. Our empirical results further demonstrate
empirically their effectiveness and their
importance in adaptation problems.

\bibliography{madap}

\begin{thebibliography}{33}
\providecommand{\natexlab}[1]{#1}
\providecommand{\url}[1]{\texttt{#1}}
\expandafter\ifx\csname urlstyle\endcsname\relax
  \providecommand{\doi}[1]{doi: #1}\else
  \providecommand{\doi}{doi: \begingroup \urlstyle{rm}\Url}\fi

\bibitem[Arndt(2004)]{arndt}
C.~Arndt.
\newblock \emph{Information Measures: Information and its Description in
  Science and Engineering}.
\newblock Signals and Communication Technology. Springer Verlag, 2004.

\bibitem[Blitzer et~al.(2007)Blitzer, Dredze, and Pereira]{blitzer_acl07}
J.~Blitzer, M.~Dredze, and F.~Pereira.
\newblock Biographies, bollywood, boom-boxes and blenders: Domain adaptation
  for sentiment classification.
\newblock In \emph{Association for Computational Linguistics (ACL)}, 2007.

\bibitem[Cover and Thomas(2006)]{CoverThomas2006}
T.~M. Cover and J.~M. Thomas.
\newblock \emph{Elements of Information Theory}.
\newblock Wiley-Interscience, 2006.

\bibitem[Donahue et~al.(2014)Donahue, Jia, Vinyals, Hoffman, Zhang, Tzeng, and
  Darrell]{decaf}
J.~Donahue, Y.~Jia, O.~Vinyals, J.~Hoffman, N.~Zhang, E.~Tzeng, and T.~Darrell.
\newblock Decaf: A deep convolutional activation feature for generic visual
  recognition.
\newblock In \emph{International Conference in Machine Learning (ICML)}, 2014.

\bibitem[Dredze et~al.(2008)Dredze, Crammer, and Pereira]{dredze_nips08}
M.~Dredze, K.~Crammer, and F.~Pereira.
\newblock Confidence-weighted linear classification.
\newblock In \emph{International Conference on Machine Learning (ICML)}, 2008.

\bibitem[Ganin and Lempitsky(2015)]{ganin_icml15}
Y.~Ganin and V.~Lempitsky.
\newblock Unsupervised domain adaptation by backpropagation.
\newblock In \emph{International Conference in Machine Learning (ICML)}, 2015.

\bibitem[Girshick et~al.(2014)Girshick, Donahue, Darrell, and Malik]{rcnn}
R.~Girshick, J.~Donahue, T.~Darrell, and J.~Malik.
\newblock Rich feature hierarchies for accurate object detection and semantic
  segmentation.
\newblock In \emph{In Proc. CVPR}, 2014.

\bibitem[Gong et~al.(2012)Gong, Shi, Sha, and Grauman]{gong_cvpr12}
B.~Gong, Y.~Shi, F.~Sha, and K.~Grauman.
\newblock Geodesic flow kernel for unsupervised domain adaptation.
\newblock In \emph{Proc. CVPR}, 2012.

\bibitem[Gong et~al.(2013{\natexlab{a}})Gong, Grauman, and Sha]{gong_iccv13}
B.~Gong, K.~Grauman, and F.~Sha.
\newblock Connecting the dots with landmarks: Discriminatively learning
  domain-invariant features for unsupervised domain adaptation.
\newblock In \emph{ICCV}, 2013{\natexlab{a}}.

\bibitem[Gong et~al.(2013{\natexlab{b}})Gong, Grauman, and Sha]{gong_nips13}
B.~Gong, K.~Grauman, and F.~Sha.
\newblock Reshaping visual datasets for domain adaptation.
\newblock In \emph{NIPS}, 2013{\natexlab{b}}.

\bibitem[Hoffman et~al.(2012)Hoffman, Kulis, Darrell, and
  Saenko]{hoffman_eccv12}
J.~Hoffman, B.~Kulis, T.~Darrell, and K.~Saenko.
\newblock Discovering latent domains for multisource domain adaptation.
\newblock In \emph{European Conference on Computer Vision (ECCV)}, 2012.

\bibitem[Hoffman et~al.(2013)Hoffman, Rodner, Donahue, Saenko, and
  Darrell]{hoffman_iclr13}
J.~Hoffman, E.~Rodner, J.~Donahue, K.~Saenko, and T.~Darrell.
\newblock Efficient learning of domain-invariant image representations.
\newblock In \emph{International Conference on Learning Representations}, 2013.

\bibitem[Huang et~al.(2006)Huang, Smola, Gretton, Borgwardt, and
  Sch{\"o}lkopf]{HuangSmolaGrettonBorgwardtScholkopf2006}
J.~Huang, A.~J. Smola, A.~Gretton, K.~M. Borgwardt, and B.~Sch{\"o}lkopf.
\newblock Correcting sample selection bias by unlabeled data.
\newblock In \emph{Advances in Neural Information Processing Systems (NIPS)},
  volume~19, pages 601--608, 2006.

\bibitem[Krizhevsky et~al.(2012)Krizhevsky, Sutskever, and Hinton]{supervision}
A.~Krizhevsky, I.~Sutskever, and G.~E. Hinton.
\newblock Image{N}et classification with deep convolutional neural networks.
\newblock In \emph{Proc. NIPS}, 2012.

\bibitem[Liao(2013)]{liao_icassp13}
H.~Liao.
\newblock Speaker adaptation of context dependent deep neural networks.
\newblock In \emph{ICASSP}, 2013.

\bibitem[Long et~al.(2015)Long, Cao, Wang, and Jordan]{long_icml15}
M.~Long, Y.~Cao, J.~Wang, and M.~I. Jordan.
\newblock Learning transferable features with deep adaptation networks.
\newblock In \emph{International Conference in Machine Learning (ICML)}, 2015.

\bibitem[Mansour et~al.(2008)Mansour, Mohri, and
  Rostamizadeh]{MansourMohriRostamizadeh2008}
Y.~Mansour, M.~Mohri, and A.~Rostamizadeh.
\newblock Domain adaptation with multiple sources.
\newblock In \emph{NIPS}, 2008.

\bibitem[Mansour et~al.(2009)Mansour, Mohri, and
  Rostamizadeh]{MansourMohriRostamizadeh2009}
Y.~Mansour, M.~Mohri, and A.~Rostamizadeh.
\newblock Multiple source adaptation and the {R}{\'e}nyi divergence.
\newblock In \emph{UAI}, pages 367--374, 2009.

\bibitem[Muandet et~al.(2013)Muandet, Balduzzi, and
  Sch{\"{o}}lkopf]{MuandetBalduzziScholkopf2013}
K.~Muandet, D.~Balduzzi, and B.~Sch{\"{o}}lkopf.
\newblock Domain generalization via invariant feature representation.
\newblock In \emph{Proceedings of {ICML}}, pages 10--18, 2013.

\bibitem[Pan and Yang(2010)]{pan_tkda2010}
S.~J. Pan and Q.~Yang.
\newblock A survey on transfer learning.
\newblock In \emph{IEEE Transactions on Knowledge and Data Engineering}, 2010.

\bibitem[R\'enyi(1961)]{Renyi1961}
A.~R\'enyi.
\newblock On measures of entropy and information.
\newblock In \emph{Proceedings of the Fourth Berkeley Symposium on Mathematical
  Statistics and Probability}, volume~1, pages 547--561, 1961.

\bibitem[Roark et~al.(2012)Roark, Sproat, Allauzen, Riley, Sorensen, and
  Tai]{roark2012opengrm}
B.~Roark, R.~Sproat, C.~Allauzen, M.~Riley, J.~Sorensen, and T.~Tai.
\newblock The opengrm open-source finite-state grammar software libraries.
\newblock In \emph{Proceedings of the ACL 2012 System Demonstrations}, pages
  61--66. Association for Computational Linguistics, 2012.

\bibitem[Saenko et~al.(2010)Saenko, Kulis, Fritz, and Darrell]{saenko_eccv10}
K.~Saenko, B.~Kulis, M.~Fritz, and T.~Darrell.
\newblock Adapting visual category models to new domains.
\newblock In \emph{Proc. ECCV}, 2010.

\bibitem[Sriperumbudur and Lanckriet(2012)]{SriperumbudurLanckriet2009}
B.~K. Sriperumbudur and G.~R.~G. Lanckriet.
\newblock A proof of convergence of the concave-convex procedure using
  {Z}angwill's theory.
\newblock \emph{Neural Computation}, 24\penalty0 (6):\penalty0 1391--1407,
  2012.

\bibitem[Taigman et~al.(2017)Taigman, Polyak, and Wolf]{taigman_iclr17}
Y.~Taigman, A.~Polyak, and L.~Wolf.
\newblock Unsupervised cross-domain image generation.
\newblock In \emph{International Conference on Learning Representations
  (ICLR)}, 2017.

\bibitem[Tao and An(1997)]{TaoAn1997}
P.~D. Tao and L.~T.~H. An.
\newblock Convex analysis approach to {DC} programming: theory, algorithms and
  applications.
\newblock \emph{Acta Mathematica Vietnamica}, 22\penalty0 (1):\penalty0
  289--355, 1997.

\bibitem[Tao and An(1998)]{TaoAn1998}
P.~D. Tao and L.~T.~H. An.
\newblock A {DC} optimization algorithm for solving the trust-region
  subproblem.
\newblock \emph{SIAM Journal on Optimization}, 8\penalty0 (2):\penalty0
  476--505, 1998.

\bibitem[Torralba and Efros(2011)]{efros_cvpr11}
A.~Torralba and A.~Efros.
\newblock Unbiased look at dataset bias.
\newblock In \emph{CVPR}, 2011.

\bibitem[Tzeng et~al.(2015)Tzeng, Hoffman, Darrell, and Saenko]{tzeng_iccv15}
E.~Tzeng, J.~Hoffman, T.~Darrell, and K.~Saenko.
\newblock Simultaneous deep transfer across domains and tasks.
\newblock In \emph{International Conference in Computer Vision (ICCV)}, 2015.

\bibitem[Xu et~al.(2014)Xu, Li, Niu, and Xu]{xu_eccv14}
Z.~Xu, W.~Li, L.~Niu, and D.~Xu.
\newblock Exploiting low-rank structure from latent domains for domain
  generalization.
\newblock In \emph{European Conference in Computer Vision (ECCV)}, 2014.

\bibitem[Yang et~al.(2007)Yang, Yan, and Hauptmann]{yang_acmm07}
J.~Yang, R.~Yan, and A.~G. Hauptmann.
\newblock Cross-domain video concept detection using adaptive svms.
\newblock \emph{ACM Multimedia}, 2007.

\bibitem[Yuille and Rangarajan(2003)]{YuilleRangarajan2003}
A.~L. Yuille and A.~Rangarajan.
\newblock The concave-convex procedure.
\newblock \emph{Neural Computation}, 15\penalty0 (4):\penalty0 915--936, 2003.

\bibitem[Zhang et~al.(2015)Zhang, Gong, and Schoelkopf]{zhang2015multi}
K.~Zhang, M.~Gong, and B.~Schoelkopf.
\newblock Multi-source domain adaptation: A causal view.
\newblock In \emph{AAAI Conference on Artificial Intelligence}, 2015.

\end{thebibliography}
{\small
\bibliographystyle{abbrvnat}}

\clearpage
\appendix

\clearpage
\section{Lower bounds for convex combination rules}
\label{app:lower_bounds}

In this section, we give lower bounds for convex combination rule, for
both squared loss and cross-entropy loss. For any $\alpha \in \Delta$, define
the convex combination rule for the {regression} and the
{probability} model as follows:
\begin{align}
g_\alpha(x)
  & = \sum_{k = 1}^p \alpha_k h_k(x), 
  & \text{(\emph{R})}\\
 g_\alpha(x, y) &= \sum_{k = 1}^p \alpha_k h_k(x, y).
  & \text{(\emph{P})}
\end{align}

\begin{lemma}[\emph{Regression model}, squared loss]
\label{lemma:lower_reg}
There is a mixture adaptation problem for which the expected squared loss of 
$g_\alpha$ is $\frac{1}{4}$.
\end{lemma}
\begin{proof}
Let $\cX = \{ a,b\}$, and $\cY=\{0,1\}$.
Consider $\scrD_0(x,y) = 1_{x=a,y=0}$, and 
$h_0(x) = 0$; 
 $\scrD_1(x,y) = 1_{x=b,y=1}$, and 
$h_1(x) = 1$.
Consider the target distribution $\sD_T = \frac{1}{2} \sD_0 + \frac{1}{2}\sD_1$. Then, for any convex combination rule $g_\alpha = \alpha h_0 + (1- \alpha) h_1 = 1- \alpha$, 
\begin{align*}
\left(\frac{1}{2}\right)^2 &= \left(\frac{1}{2} \alpha + 
\frac{1}{2}(1-\alpha)\right)^2
 = \Bigg(\sum_{(x,y)\in \cX\times \cY} \sD_T(x,y) |g_\alpha(x) - y| \Bigg)^2 \\
& \le \sum_{(x,y)\in \cX\times\cY} \sD_T(x,y) (g_\alpha(x)-y)^2  = \cL(\sD_T,g_\alpha).
\end{align*}
\end{proof}
Note that the hypotheses $h_0$ and $h_1$ have \emph{zero} error on their own domain, i.e. $\e=0$. However,
\emph{no} convex combination
rule will perform well on the target distribution $\sD_T$.

\begin{lemma}[\emph{Probability Model}, cross-entropy loss]
\label{lemma:lower_log}
There is a mixture adaptation problem for which the expected cross-entropy loss of 
$g_\alpha$ is $\log(p)$.
\end{lemma}
\begin{proof}
Let $\cX = \{ x_1,\dots,x_k\}$, and $\cY=\{ y_1,\dots,y_k\}$.
Consider $\scrD_k(x,y) = 1_{x=x_k,y=y_k}$, and 
$h_k(x,y) = 1_{y=y_k}$.
 Consider the largest 
cross-entropy loss of $g_\alpha$ on any target mixture $\scrD_\lambda(x,y)$:
 \begin{align*}
\max_{\lambda \in \Delta} \cL(\scrD_\lambda, g_\alpha)
& =  \max_{\lambda \in \Delta} \sum_{k = 1}^p -\lambda_k \log(\alpha_k) 
 =  \max_{k\in [p]} \left[ -\log (\alpha_k)\right].
 \end{align*}
 Choosing $\alpha \in \Delta$ to minimize that adversarial loss gives
 \begin{equation*}
  \min_{\alpha \in \Delta} \max_{k\in [p]} \left[-\log(\alpha_k)\right]=\log(p).
 \end{equation*}
 Therefore any convex combination rule $g_\alpha$ 
 incurs at least a loss of $\log(p)$.
  \end{proof}
Again, the base hypotheses $h_k$s have \emph{zero} error on their own domain, yet there is no convex combination rule that is robust against any target mixture.

\clearpage
\section{Theoretical analysis for the stochastic scenario}
\label{app:theory}
In this section, we give a series of theoretical results for the
general stochastic scenario with their full proofs.  
We will separate the proofs for the regression model (Appendix~\ref{app:reg_theory}) and the probability model (Appendix~\ref{app:prob_theory}), since the definitions of the distribution weighted combination are different in the two models.

\subsection{Regression model}
\label{app:reg_theory}
The proofs for the regression model (R) are presented in the following order: 
we first assume the conditional probabilities are the same across source domains, and prove Lemma~\ref{lemma:brouwer}; using
that, we prove Corollary~\ref{th:mixture} and Corollary~\ref{th:estimate}. Finally, we relax the assumption of same conditionals, and prove Theorem~\ref{th:distinct}, which a stronger version of Theorem~\ref{th:distinct_mixture}.

Our proofs make use of the following Fixed-Point Theorem of Brouwer.
\begin{theorem}
  For any compact and convex non-empty set $C \subset \Rset^p$ and any
  continuous function $f\colon C\rightarrow C$, there is a point
  $x \in C$ such that $f(x) = x$.
\end{theorem}

\begin{replemma}{lemma:brouwer}
For any $\eta, \eta' > 0$, there exists $z \in \Delta$, with $z_k \neq
0$ for all $k \in [p]$, such that the following holds for the
distribution-weighted combining rule $h_z^\eta$:
\begin{equation}
\forall k \in [p], \quad \cL(\sD_k, h_z^\eta) \leq \sum_{j = 1}^p z_j \cL(\sD_j, h_z^\eta)  + \eta'.
\end{equation}
\end{replemma}

\begin{proof}
  Consider the mapping $\Phi \colon \Delta \to \Delta$ defined for all
  $z \in \Delta$ by
\begin{equation*} 
[\Phi(z)]_k = \frac{z_k \, \cL(\sD_k, h_z^\eta) +
      \frac{\eta'}{p} }{\sum_{j = 1}^p z_j \cL(\sD_j, h_z^\eta) + \eta'}.
\end{equation*}
$\Phi$ is continuous since $\cL(\sD_k, h_z^\eta)$ is a continuous
function of $z$ and since the denominator is positive ($\eta' > 0$).
Thus, by Brouwer's Fixed Point Theorem, there exists $z \in \Delta$
such that $\Phi(z) = z$. For that $z$, we can write
\begin{equation*}
z_k = \frac{z_k \, \cL(\sD_k, h_z^\eta) +
      \frac{\eta'}{p} }{\sum_{j = 1}^p z_j \cL(\sD_j, h_z^\eta) + \eta'},
\end{equation*}
for all $k \in [p]$.  Since $\eta'$ is positive, we must have $z_k
\neq 0$ for all $k$. Dividing both sides by $z_k$ gives $\cL(\sD_k,
h_z^\eta) = \sum_{j = 1}^p z_j \cL(\sD_j, h_z^\eta) + \eta' -
\frac{\eta'}{p z_k} \leq \sum_{j = 1}^p z_j \cL(\sD_j, h_z^\eta) +
\eta'$, which completes the proof.
\end{proof}

\begin{repcorollary}{th:mixture}
Assume the conditional probability $\sD_k(y|x)$ does not depend on $k$. Let $\sD_\lambda$ be an arbitrary mixture of source domains,   $\lambda \in \Delta$. For any $\delta > 0$, there exists $\eta > 0$ and $z\in \Delta$, such
that $\cL(\sD_\lambda,h_z^\eta)\leq \e +\delta$.
\end{repcorollary}

\begin{proof}
  We first upper bound, for an arbitrary $z \in \Delta$, the expected
  loss of $h_z^\eta$ with respect to the mixture distribution $\sD_z$
  defined using the same $z$, that is
  $\cL(\sD_z, h_z^\eta) = \sum_{k = 1}^p z_k \cL(\sD_k, h_z^\eta)$.
  By definition of $h_z^\eta$ and $\sD_z$, we can write
\begin{align*}
     \cL(\sD_z, h_z^\eta) & = \sum_{(x, y) } \sD_z(x,y) L(h_z^\eta(x), y)\\
    & = \sum_{(x, y)} \sD_z(x, y) L \mspace{-2mu} \left(\mspace{-1mu}
  \sum_{k = 1}^p \frac{z_k
    \msD_k(x) + \eta  \frac{\msU(x)}{p}}{\msD_z(x) + \eta
    \msU(x)} h_k(x), y \right) \mspace{-4mu}.
\end{align*}
By convexity of $L$, this implies that
\begin{align*}
     \cL(\sD_z, h_z^\eta) 
    & \leq \sum_{(x, y)} \sD_z(x, y) \sum_{k = 1}^p \frac{z_k
\msD_k(x) + \eta  \frac{\msU(x)}{p}}{\msD_z(x) + \eta
\msU(x)} L\big(h_k(x), y \big)\\
& \leq \sum_{(x, y)} \sD_z(y | x) \msD_z(x) \sum_{k = 1}^p \frac{z_k
    \msD_k(x) + \eta  \frac{\msU(x)}{p}}{\msD_z(x) + \eta
    \msU(x)} L\big(h_k(x), y \big)\\
    & \leq \sum_{(x, y)} \sD_z(y | x) \sum_{k = 1}^p \bigg( z_k
    \msD_k(x) + \eta  \frac{\msU(x)}{p} \bigg) L\big(h_k(x), y \big).
\end{align*}
Next, observe that
$\sD_{z}(y|x)=\sum_{k = 1}^{p}\frac{z_{k}\msD_{k}(x)}{\msD_z(x)}
\sD_{k}(y|x)=\sD_k(y|x)$ for any $k \in [p]$ 
since by assumption $\sD_{k}(y | x)$ does not
depend on $k$. Thus,
\begin{align*}
    \cL(\sD_z, h_z^\eta)
& \leq \sum_{(x, y)} \sD_z(y | x) \sum_{k = 1}^p \bigg( z_k
    \msD_k(x) + \eta  \frac{\msU(x)}{p} \bigg) L\big(h_k(x), y \big)\\
    &  =  \sum_{(x, y)} \mspace{-2mu}\sum_{k = 1}^p \bigg(\mspace{-2mu} z_k
    \sD_k(x,y) + \eta  \sD_k(y | x) \frac{\msU(x)}{p} \mspace{-2mu}\bigg)
  \mspace{-2mu} L\big(h_k(x), y \big)\\
  & = \sum_{k = 1}^p z_k \cL(\sD_k, h_k) + \frac{\eta}{p} \sum_{k = 1}^p \sum_{(x,
  y)} \sD_k(y | x) \msU(x)
  L\big(h_k(x), y \big)\\
  &  \leq  \sum_{k = 1}^p z_k \cL(\sD_k, h_k) + \eta M 
   \leq  \sum_{k = 1}^p z_k \e + \eta M = \e + \eta M.
\end{align*}
Now, choose $z \in \Delta$ as in the statement of
Lemma~\ref{lemma:brouwer}.
Then,
the following holds for any mixture distribution $\sD_\lambda$:
\begin{align*}
\cL(\sD_\lambda,h_z^\eta)
&= \sum_{k = 1}^p \lambda_k \cL(\sD_k,h_z^\eta) 
 \leq \sum_{k = 1}^p \lambda_k \big( \cL(\sD_z,h_z^\eta) +\eta' \big)\\
 & = \cL(\sD_z,h_z^\eta) +\eta'
 \leq \e + \eta M + \eta'.
\end{align*}
Setting $\eta = \frac{\delta}{2M}$ and $\eta' = \frac{\delta}{2}$
concludes the proof.
\end{proof}

Next, we introduce a useful Corollary and give its proof.

\begin{corollary}\label{th:arbitrary}
Let $\sD_T$ be an arbitrary target distribution.
For any $\delta > 0$, there exists
$\eta > 0$ and $z \in \Delta$, such that
the following inequality holds for any $\alpha > 1$:
\begin{equation*}
\cL(\sD_T, h_z^\eta) 
\leq \Big[(\e + \delta) \, \sfd_\alpha(\sD_T \parallel \cD)
\Big]^{\frac{\alpha - 1}{\alpha}} M^{\frac{1}{\alpha }}.
\end{equation*}
\end{corollary}

\begin{proof}
  For any hypothesis $h\colon \cX \to \cY $ and any distribution
  $\sD$, by H\"older's inequality, the following holds:
\begin{align*}
    \cL(\sD_T, h) 
    & = \sum_{(x, y) \in \cX \times \cY} \sD_T(x, y) L(h(x), y)\\
& = \sum_{(x, y) \in \cX \times \cY} \left [\frac{\sD_T(x, y)}{\sD(x,
  y)^{\frac{\alpha - 1}{\alpha}}}\right] \left[ \sD(x,
  y)^{\frac{\alpha - 1}{\alpha}} L(h(x), y) \right]\\
& \leq \left [\sum_{(x, y)} \frac{\sD_T(x, y)^\alpha}{\sD(x,
  y)^{\alpha - 1}}\right]^{\frac{1}{\alpha}} \left[ \sum_{(x, y)} \sD(x,
  y) L(h(x), y)^{\frac{\alpha}{\alpha - 1}} \right]^{\frac{\alpha -
  1}{\alpha}} .
\end{align*}
Thus, by definition of $\sfd_\alpha$, for any $h$ such that
$L(h(x), y) \leq M$ for all $(x, y)$, we can write
\begin{align*}
     \cL(\sD_T, h) 
& \leq  \sfd_\alpha(\sD_T \mspace{-4mu}\parallel \mspace{-4mu}\sD)^{\frac{\alpha -
  1}{\alpha}} \mspace{-8mu} \left[ \sum_{(x, y)} \sD(x,
  y) L(h(x), y)^{\frac{\alpha}{\alpha - 1}} \mspace{-4mu} \right]^{\mspace{-8mu}\frac{\alpha -
  1}{\alpha}}\\
& = \sfd_\alpha(\sD_T \mspace{-4mu}\parallel \mspace{-4mu}\sD)^{\frac{\alpha - 1}{\alpha}}
  \mspace{-8mu} \left[ \sum_{(x, y)} \mspace{-4mu}\sD(x,
  y) L(h(x), y) L(h(x), y)^{\frac{1}{\alpha - 1}} \mspace{-4mu} \right]^{\mspace{-8mu}\frac{\alpha -
  1}{\alpha}}\\
&\leq \sfd_\alpha(\sD_T \mspace{-4mu}\parallel \mspace{-4mu}\sD)^{\frac{\alpha - 1}{\alpha}} \left[ \sum_{(x, y)} \mspace{-4mu}\sD(x,
  y) L(h(x), y) M^{\frac{1}{\alpha - 1}} \right]^{\mspace{-8mu}\frac{\alpha -
  1}{\alpha}}\\
&\leq \Big[ \sfd_\alpha(\sD_T \mspace{-4mu}\parallel \mspace{-4mu}\sD) \, \cL(\sD, h)
  \Big]^{\mspace{-6mu}\frac{\alpha - 1}{\alpha}} M^{\frac{1}{\alpha }}.
\end{align*}
Now, by Corollary~\ref{th:mixture}, 
there exists $z \in \Delta$ and
$\eta > 0$ such that $\cL(\sD, h_z^\eta) \leq \e + \delta$ for any
mixture distribution $\sD \in \cD$. Thus, in view of the previous
inequality, we can write,for any $\sD \in \cD$,
\begin{equation*}
\cL(\sD_T, h_z^\eta) \leq \Big[(\e + \delta) \, \sfd_\alpha(\sD_T \parallel \sD)
 \Big]^{\frac{\alpha - 1}{\alpha}} M^{\frac{1}{\alpha }}.
\end{equation*}
Taking the infimum of the right-hand side over all $\sD \in \cD$
completes the proof.
\end{proof}

\begin{corollary}
\label{th:estimate}
Let $\sD_T$ be an arbitrary target distribution.
Then, for any $\delta > 0$, there exists $\eta > 0$ and
$z \in \Delta$, such that the following inequality holds for any
$\alpha > 1$:
\begin{equation*}
\cL(\sD_T, \h h_z^\eta) 
\leq \Big[(\h \e + \delta) \, \sfd_\alpha(\sD_T \parallel \h \cD)
\Big]^{\frac{\alpha - 1}{\alpha}} M^{\frac{1}{\alpha }},
\end{equation*}
where $\h \e = \max_{k \in [p]} \Big[\e  \, \sfd_\alpha(\h \sD_k \parallel \sD_k)
\Big]^{\frac{\alpha - 1}{\alpha}} M^{\frac{1}{\alpha }}$, 
and $\h \cD = \left\{ \sum_{k = 1}^p \lambda_k \h \sD_k\colon \lambda \in \Delta \right\}$.
\end{corollary}
\begin{proof}
By the first part of the proof of Corollary~\ref{th:arbitrary},
for any $k \in [p]$ and $\alpha > 1$, the following
inequality holds:
\begin{align*}
\cL(\h \sD_k, h_k) 
& \leq \Big[\sfd_\alpha(\h \sD_k \parallel  \sD_k) \, \cL(\sD_k, h_k)
\Big]^{\frac{\alpha - 1}{\alpha}} M^{\frac{1}{\alpha }}\\
& \leq \Big[\e \, \sfd_\alpha(\h \sD_k \parallel  \sD_k) 
\Big]^{\frac{\alpha - 1}{\alpha}} M^{\frac{1}{\alpha }} \leq \h \e.
\end{align*}
We can now apply the result of Corollary~\ref{th:arbitrary} 
(with
$\h \e$ instead of $\e$ and $\h \sD_k$ instead of $\sD_k$).
In view that, there exists $\eta > 0$ and $z \in \Delta$
such that
\begin{equation*}
\cL(\sD_T, h_z^\eta) 
\leq \Big[(\h \e + \delta) \, \sfd_\alpha(\sD_T \parallel \h \sD)
\Big]^{\frac{\alpha - 1}{\alpha}} M^{\frac{1}{\alpha }},
\end{equation*}
for any distribution $\h \sD$ in the family $\h \cD$. Taking the
infimum over all $\h \sD$ in $\h \cD$ completes the proof.
\end{proof}

This result shows that there exists a predictor $\h h_z^\eta$ based on
the estimate distributions $\h \sD_k$ that is $\h \e$-accurate with
respect to any target distribution $\sD_T$ whose R\'enyi divergence
with respect to the family $\h \cD$ is not too large
($\sfd_\alpha(\sD_T \parallel \h \cD)$ close to $1$).  Furthermore,
$\h \e$ is close to $\e$, provided that $\h \sD_k$s are good estimates
of $\sD_k$s (that is $\sfd_\alpha(\h \sD_k \parallel \sD_k)$ close to
$1$).

Corollary~\ref{th:estimate} used R\'enyi divergence
in both directions: $\sfd_\alpha(\sD_T \parallel \h \cD)$
requires $\text{Supp}(\sD_T) \subseteq \text{Supp}(\h\cD)$,
and $\sfd_\alpha(\h \sD_k \parallel \sD_k)$ requires 
$\text{Supp}(\h\sD_k) \subseteq \text{Supp}(\sD_k), k\in [p]$.
In our experiments in Section~\ref{sec:eval}, 
we used bigram language model for sentiment analysis, and 
kernel density estimation with a Gaussian kernel for object recognition. Both density estimation methods fulfill these requirements. 

Finally we prove our main result Theorem~\ref{th:distinct_mixture} under the regression model (R). We first prove a stronger version for Theorem~\ref{th:distinct_mixture}, next we show that it will coincide with Theorem~\ref{th:distinct_mixture} under the assumption that $\msD_T \in \cD^1$.
\begin{theorem}\label{th:distinct}
Let $\sD_T$ be an arbitrary target distribution. Then, for any $\delta >0$, 
there exists $\eta >0$ and $z \in \Delta$ such that the following inequality
holds for any $\alpha>1$:
\begin{align*}
\cL(\sD_T, h_z^\eta) &\leq \Big[ (\e_{T} + \delta) \sfd_\alpha (\sD_T \parallel \sD_{P,T})\Big]^
{\frac{\alpha-1}{\alpha}} M^{\frac{1}{\alpha}} & (R),
\end{align*}
where 
$$    \e_{T}=\max_{k\in [p]} \Big[\mathbb{E}_{\msD_k(x)} \sfd_{\alpha}\left(
\sD_T(\cdot | x) \parallel \sD_k(\cdot | x)\right)^{\alpha-1}\Big]^{\frac{1}{\alpha}}
\e^{\frac{\alpha-1}{\alpha}}M^{\frac{1}{\alpha}},$$
and $\sD_{k,T}(x,y)=\msD_k(x)\sD_T(y|x)$,
$\sD_{P,T} = \big\{ \sum_{k = 1}^p \lambda_k \sD_{k,T}, \lambda \in \Delta \big\}$.
\end{theorem}

\begin{proof}
For any domain $k$, by H\"older's inequality, the following holds:
\begin{align*}
    \cL(\sD_{k,T},h_k) 
    & = \sum_{x,y} \msD_k(x) \sD_T(y|x) L(h_k, x,y)  \\
& =\sum_{x}\msD_k(x) \sum_y\left[\frac{\sD_T(y|x)}{\sD_k(y|x)^
{\frac{\alpha-1}{\alpha}}}\right]
\left[\sD_k(y|x)^{\frac{\alpha-1}{\alpha}}L(h_k,x,y)\right] \\
 &\le \sum_{x}\msD_k(x)  \sfd_{\alpha}(x;T,k)^{\frac{\alpha-1}{\alpha}} 
\Big[\sum_{y}\sD_k(y| x)L(h_k,x,y)^{\frac{\alpha}{\alpha-1}}\Big]^{\frac{\alpha-1}{\alpha}}
\end{align*}
where, for simplicity, we write 
$\sfd_{\alpha}(x;T,k)  = 
 \sfd_{\alpha}\left(\sD_T(\cdot| x)\parallel \sD_k(\cdot| x)\right)$. 
 Using the fact that the loss is bounded and H\"older's inequality again, 
  \begin{align*}
       \cL(\sD_{k,T},h_k)   &\le 
  \sum_{x} \msD_k(x)^{\frac{1}{\alpha}}
\sfd_{\alpha}(x;T,k)^{\frac{\alpha-1}{\alpha}} 
\left[\sum_{y}
\sD_k(x,y)L(h_k,x,y)\right]^{\frac{\alpha-1}{\alpha}}
  M^{\frac{1}{\alpha}}\\
&\le \left[\sum_{x}\msD_k(x)\sfd_{\alpha}(x;T,k)^{\alpha-1}\right]^{\frac{1}{\alpha}}
\left[\sum_{x,y} \sD_k(x,y)L(h_k,x,y)\right]^{\frac{\alpha-1}{\alpha}} M^{\frac{1}{\alpha}}  \\
& \le  \Big[\mathbb{E}_{\msD_k} \sfd_{\alpha}(x;T,k)^{\alpha-1}\Big]^{\frac{1}{\alpha}}
\e^{\frac{\alpha-1}{\alpha}}  M^{\frac{1}{\alpha}}  \le \e_{T}.
\end{align*}
We can now apply the result of Corollary~\ref{th:arbitrary}, 
with $\e_{T}$ instead of $\e$ and $\sD_{k,T}$ instead of $\sD_k$. 
This completes the proof.
 \end{proof}
When $\msD_T \in \cD^1$, $\sD_T \in \sD_{P,T}$, thus by the definition of R\'enyi divergence, $\sfd_\alpha (\sD_T \parallel \sD_{P,T})=1$. Theorem~\ref{th:distinct} coincides with Theorem~\ref{th:distinct_mixture} in this case.  
 
\subsection{Probability model}
\label{app:prob_theory}

In this section, we first present a series of general theoretical
results for the {probability model} (P) in the same order as in Appendix~\ref{app:reg_theory} . Many of the them are similar to those for the {regression model}, except that we do not assume anything about the conditional probabilities throughout the proofs.  In several instances, the proofs are syntactically the same as their counterparts in the {regression model} (R). In such cases, we do not reproduce them.

\begin{replemma}{lemma:brouwer}
For any $\eta, \eta' > 0$, there exists $z \in \Delta$, with $z_k \neq
0$ for all $k \in [p]$, such that the following holds for the
distribution-weighted combining rule $h_z^\eta$:
\begin{equation}
\forall k \in [p], \quad \cL(\sD_k, h_z^\eta) \leq \sum_{j = 1}^p z_j \cL(\sD_j, h_z^\eta)  + \eta'.
\end{equation}
\end{replemma}

\begin{proof}
    The proof is syntactically the same as that for the regression model.
    \end{proof}

\begin{repcorollary}{th:mixture}
For any $\delta > 0$, there exists $\eta > 0$ and $z\in \Delta$, such
that $\cL(\sD_\lambda,h_z^\eta)\leq \e +\delta$ \ for any mixture
parameter $\lambda \in \Delta$.
\end{repcorollary}
\begin{proof}
    Modifying the proof of Corollary~\ref{th:mixture} for the {regression model} gives

\begin{align*}
     \cL(\scrD_z, h_z^\eta)
& = \sum_{(x, y) \in \cX \times \cY} \scrD_z(x,y) L(h_z^\eta(x, y))\\
& = \sum_{(x, y)} \scrD_z(x, y) L \mspace{-2mu} \left(\mspace{-1mu} \sum_{k = 1}^p \frac{z_k
    \scrD_k(x, y) + \eta  \frac{\scrU(x, y)}{p}}{\scrD_z(x, y) + \eta
    \scrU(x, y)} h_k(x, y) \mspace{-4mu} \right) \mspace{-4mu}.
\end{align*}
By convexity of $L$, this implies that
\begin{align*}
    &\cL(\scrD_z, h_z^\eta)
     \leq \sum_{(x, y)} \scrD_z(x, y) \sum_{k = 1}^p \frac{z_k
    \scrD_k(x, y) + \eta  \frac{\scrU(x, y)}{p}}{\scrD_z(x, y) + \eta
    \scrU(x, y)} L\big(h_k(x, y)\big).
\end{align*}
 Next, since $\frac{\scrD_z(x, y)}{\scrD_z(x, y) + \eta
    \scrU(x, y)} \leq 1$, the following holds:
\begin{align*}
    \cL(\scrD_z, h_z^\eta)
& \leq \sum_{(x, y)}
\Big(\sum_{k = 1}^p \big(z_k \scrD_k(x,y) + \tfrac{\eta \scrU(x,y)}{p} \big)
L(h_k(x,y)) \Big)\\
& = \sum_{k = 1}^p z_k \cL(\scrD_k, h_k) + \frac{\eta}{p} \sum_{k = 1}^p
\cL(\scrU, h_k)\\
& \leq \sum_{k = 1}^p z_k \e + \eta M
= \e+ \eta M .
\end{align*}
Now choose $z \in \Delta$ as in the statement of
Lemma 4a.
Then, the following holds for any mixture distribution $\scrD_\lambda$:
\begin{align*}
    \cL(\scrD_\lambda,h_z^\eta)
    &= \sum_{k = 1}^p \lambda_k \cL(\scrD_k,h_z^\eta) 
 \leq \sum_{k = 1}^p \lambda_k (\cL(\scrD_z,h_z^\eta) +\eta')\\
 & = \cL(\scrD_z,h_z^\eta) +\eta'
 \leq \e + \eta M + \eta'.
\end{align*}

Setting $\eta = \frac{\delta}{2M}$ and $\eta' = \frac{\delta}{2}$
concludes the proof.

\end{proof}

Since we do not assume the conditional probabilities are the same across domains,  we can directly prove Theorem~\ref{th:distinct} for the conditional probability model (P), which coincides with Theorem~\ref{th:distinct_mixture} when $\sD_T \in \cD$.

\begin{reptheorem}{th:distinct}
Let $\sD_T$ be an arbitrary target distribution.
For any $\delta > 0$, there exists
$\eta > 0$ and $z \in \Delta$, such that
the following inequality holds for any $\alpha > 1$:
\begin{align*}
\cL(\sD_T, h_z^\eta) 
& \leq \Big[(\e + \delta) \, \sfd_\alpha(\sD_T \parallel \cD)
\Big]^{\frac{\alpha - 1}{\alpha}} M^{\frac{1}{\alpha }} & (P).
\end{align*}
\end{reptheorem}
\begin{proof}
The proof is syntactically the same as that of Corollary~\ref{th:arbitrary} for the regression model.
\end{proof}

\begin{repcorollary}{th:estimate}
Let $\sD_T$ be an arbitrary target distribution.
Then, for any $\delta > 0$, there exists $\eta > 0$ and
$z \in \Delta$, such that the following inequality holds for any
$\alpha > 1$:
\begin{equation*}
\cL(\sD_T, \h h_z^\eta) 
\leq \Big[(\h \e + \delta) \, \sfd_\alpha(\sD_T \parallel \h \cD)
\Big]^{\frac{\alpha - 1}{\alpha}} M^{\frac{1}{\alpha }},
\end{equation*}
where $\h \e = \max_{k \in [p]} \Big[\e  \, \sfd_\alpha(\h \sD_k \parallel \sD_k)
\Big]^{\frac{\alpha - 1}{\alpha}} M^{\frac{1}{\alpha }}$, 
and $\h \cD = \left\{ \sum_{k = 1}^p \lambda_k \h \sD_k\colon \lambda \in \Delta \right\}$.
\end{repcorollary}
\begin{proof}
The proof is syntactically the same as that of Corollary~\ref{th:estimate} for the regression model.
\end{proof}

\clearpage
\section{Specific theoretical analysis for the cross-entropy loss}
\label{app:log-loss}
Next, we give a specific theoretical analysis for the case of the
cross-entropy loss. This is needed since the cross-entropy loss assumes normalized hypotheses. Thus, we are giving guarantees for the performance of normalized distribution-weighted predictor.

We will first assume that the conditional probability of the output
labels is the same for all source domains, that is, for any $(x, y)$,
$\scrD_k(y|x)$ is independent of $k$. 

\begin{reptheorem}{th:normalized}
Assume there exists $\mu>0$ such that $\scrD_k(x,y) \ge \mu \scrU(x,y)$ 
for all $k\in[p]$ and $(x,y)\in \cX \times \cY$. Then, for any $\delta > 0$, there exists
$\eta > 0$ and $z\in \Delta$, such that $\cL(\scrD_\lambda,\overline h_z^\eta)\leq
\e +\delta$ \ for any mixture parameter $\lambda \in \Delta$.
\end{reptheorem}

\begin{proof}
    By the proof of Corollary~\ref{th:mixture} for the probability model, 
    for any mixture distribution $\scrD_\lambda$:
\begin{equation*}
\cL(\scrD_\lambda,h_z^\eta) \leq \e + \eta M + \eta',
\end{equation*}
for some $\eta >0, \eta' >0$. For any $x\in \cX$,
\begin{align*}
       \overline h_z^\eta(x) 
    &= \sum_{y\in \cY} \sum_{k = 1}^p\frac{z_k \scrD_k(x, y) + \frac{\eta \scrU(x,y)}{p}}{
\scrD_z(x,y) + \eta \scrU(x,y) } h_k(x, y) \\
&\leq \sum_{y\in \cY} \sum_{k = 1}^p\frac{z_k \scrD_k(x, y) + \frac{\eta \scrU(x,y)}{p}}{
 \scrD_z(x, y) } h_k(x, y) \\
& = 1 + \eta \left[\frac{1}{p}\sum_{y\in \cY} \sum_{k = 1}^p\frac{\scrU(x,y) }{
 \scrD_z(x, y) } h_k(x, y) \right].
\end{align*}
By assumption, $\scrD_k(x,y) \ge \mu \scrU(x,y)$ for any $(x,y)$.
Therefore $\scrD_z(x,y) \ge \mu \scrU(x,y)$ for any $z\in \Delta$.
Since $0\le h_k(x,y)\le 1$, $\overline h_z^\eta(x)$ is upper bounded by
\begin{align*}
    \overline h_z^\eta(x) &\leq  1 + 
    \eta \left[\frac{1}{p}\sum_{y\in \cY} \sum_{k = 1}^p\frac{\scrU(x,y) }{
\scrD_z(x, y) } h_k(x, y) \right] 
\leq 1+ \frac{\eta |\cY|}{\mu}.
\end{align*} 
It follows that
\begin{align*}
    \cL(\scrD_\lambda,\overline h_z^\eta) 
  & = \cL(\scrD_\lambda,h_z^\eta)  
    + \mathbb{E}_{\scrD_\lambda(x)}[\log(\overline h_z^\eta(x))] 
    \leq  \e + \eta M + \eta' + \log \left(1 +  \frac{\eta |\cY|}{\mu}\right) \\
  &  \leq \e + \eta \left(M +\frac{ |\cY|}{\mu}\right) + \eta' . 
\end{align*}
Setting $\eta = \frac{\delta}{2\left(M +\frac{ |\cY|}{\mu}\right)}$ 
and $\eta' = \frac{\delta}{2}$ concludes the proof.
\end{proof}

The analysis above depends on the key assumption that the conditional
distributions $\scrD_k(y | x)$ are independent of $k$.  When this
assumption does not hold, we can show that there is a lower bound of
$\log(p)$ on the generalization error
$\cL(\scrD_\lambda,\overline h_z^\eta)$. However, this lower bound
coincides with that of convex combination rule (Lemma~\ref{lemma:lower_log}).
In that case, one can use the following 
marginal distribution-weighted combination instead:
\begin{equation}
\label{eq:marginal_h}
  \wt h_z^\eta(x, y) = \sum_{k = 1}^p\frac{z_k \msD_k(x) + \eta \, \frac{\msU(x)}{p}}{
    \sum_{j = 1}^p z_j \msD_j(x) + \eta \, \msU(x)} h_k(x, y),
\end{equation}
where $\msD_k(x)$ is the marginal distribution over $\cX$,
$\msD_k(x) = \sum_{y\in\cY}\scrD_k(x,y)$,
and $\msU(x)$ is a uniform distribution over $\cX$.
Observe that $\wt h_z^\eta(x,y)$ is already normalized. 

One can modify Theorem~\ref{th:distinct}
to obtain generalization guarantees
for $\wt h_z^\eta$ under distinct conditional probabilities assumption. 
Let $\sD_T(x,y)$, $\e_{T}$ and $\sD_{P,T}$ be defined as before.

\begin{theorem}
\label{th:normalized-bis}
Let $\sD_T$ be an arbitrary target distribution. Then, for any $\delta >0$, 
there exists $\eta >0$ and $z \in \Delta$ such that the following inequality
holds for any $\alpha>1$:
\begin{equation*}
\cL(\sD_T, \wt h_z^\eta) \le \Big[ (\e_{T} + \delta) \sfd_\alpha (\sD_T \parallel \sD_{P,T})\Big]^
{\frac{\alpha-1}{\alpha}} M^{\frac{1}{\alpha}}.
\end{equation*}
\end{theorem}

\begin{proof}
    The proof is syntactically the same as that of Theorem~\ref{th:distinct}.
\end{proof}

Finally, we can extend Theorem~\ref{th:normalized} and Theorem~\ref{th:normalized-bis} 
to the case where only estimate distributions $\h\sD_k$s are available, and the predictor $\overline{ \h h_z^\eta}$
and $\widetilde{ \h h_z^\eta}$ based on the estimates $\h\sD_k$ still admit favorable guarantees. 
The results and proofs are similar to proving Corollary~\ref{th:estimate} from Corollary~\ref{th:arbitrary} in the regression model,  
thus omitted here. 

\clearpage
\section{DC-decomposition}
\label{app:dcd}

In this section we give the full proofs for the DC-decompositions
presented in Section~\ref{sec:dc}.

\subsection{Regression model}

\begin{repproposition}{prop:dc-regression}
Let $L$ be the squared loss. Then, for any $k \in [p]$,
$\cL(\sD_k, h_z^\eta) - \cL(\sD_z, h_z^\eta) = u_k(z) - v_k(z)$, where
$u_k$ and $v_k$ are convex functions defined for all $z$ by
\begin{align*}
u_k(z) &= \cL\left(\sD_k+\eta\msU \sD_k(\cdot | x),h_z^\eta \right) 
       -2M\sum_{x}(\msD_k+\eta\msU)(x)\log K_z(x),\\
v_k(z) &= \cL\left(\sD_z+\eta\msU \sD_k(\cdot | x),h_z^\eta \right) 
       -2M\sum_{x}(\msD_k+\eta\msU)(x)\log K_z(x).
\end{align*}
\end{repproposition}
\begin{proof}
First, observe that 
$(h_z^\eta(x) -y)^2  = f_z(x,y)-g_z(x)$,
where for every $(x,y)\in \cX \times \cY$, 
$f_z$ and $g_z$ are convex functions defined for all $z$:
\begin{align*}
    f_z(x,y)  &= \left(h_{z}^\eta(x)-y\right)^2-2M\log K_z(x),\\
g_z(x) &= -2M\log K_z(x).
\end{align*}
This is true because the Hessian matrix of $f_z$ and $g_z$ are
\begin{align*}
    H_{f_z} &=  \frac{2}{K_z^2} \left[h_{D,z} h_{D,z}^T + \left(M - (y-h_z^\eta)^2\right) DD^T \right] ,\\
    H_{g_z} &=  \frac{2M}{K_z^2}  DD^T,  
\end{align*}
where $h_{D,z}$ is a $p$-dimensional vector defined as
 $[h_{D,z}]_k = \sD_k(h_k + y - 2 h_z^\eta)$ for $k\in [p]$,
and $D = (\sD_1, \sD_2, \dots,\sD_p)^T$. 
Using the fact that $ M \ge (y-h_z^\eta)^2$, $H_{f_z}$ and $H_{g_z}$ are 
positive semidefinite matrices, therefore $f_z, g_z$ are convex functions of $z$.

Thus, $u_k(z)= \sum_{(x,y)} (\msD_k+\eta \msU)(x)\sD_k(y|x) f_z(x,y)$
is convex. Similarly, we can write the second term of $v_k(z)$ as 
$\sum_x (\msD_k + \eta \msU)(x) g_z(x)$, it is convex.
Using the notation previously defined, we can write the first term of $v_k(z)$ as
\begin{align*}
    &\cL(\sD_z+\eta\msU \sD_k(\cdot | x), h_z^\eta)
= \sum_{x} \frac{J_z(x)^2}{K_z(x)}  -2\mathbb{E} (y|x) J_z(x) + 
 \mathbb{E}(y^2|x) K_z(x).
\end{align*}
The Hessian matrix of $J_z^2/K_z$ is
\begin{equation*}
\nabla_z^2 \left(\frac{J_z^2}{K_z}\right)
 = \frac{1}{K_z}(h_D -h_z^\eta D)(h_D-h_z^\eta D)^T
\end{equation*}
where $h_D=(h_1 \sD_1,h_2 \sD_2,\dots,h_p \sD_p )^T$
and $D = (\sD_1, \sD_2, \dots,\sD_p)^T$. Thus $J_z^2/K_z$ is
convex. $-2\mathbb{E} (y|x) J_z(x) +  \mathbb{E}(y^2|x) K_z(x)$ 
is an affine function of $z$ and is therefore convex. 
Therefore the first term of $v_k(z)$ is convex, which completes the proof.  
\end{proof}

\subsection{Probability model}

\begin{repproposition}{prop:dc-probability}
  Let $L$ be the cross-entropy loss. Then, for $k \in [p]$,
  $\cL(\scrD_k, h_z^\eta) - \cL(\scrD_z, h_z^\eta) = u_k(z) - v_k(z)$,
  where $u_k$ and $v_k$ are convex functions defined for all $z$
  by
\begin{align*}
    u_k(z) & = -\sum_{x, y}  \big[ \scrD_k(x, y) + \eta \, \scrU(x,  y) \big] \log J_z(x, y),\\
    v_k(z) & = \sum_{x, y} K_z(x, y) \log \left[  \frac {K_z(x, y)}{J_z(x, y)} \right]    \\
           &- \left[ \scrD_k(x, y) + \eta \, \scrU(x,  y) \right] \log K_z(x, y).
\end{align*}
\end{repproposition}
\begin{proof}
  Using the notation previously introduced, we can now write
\begin{align*}
&\cL(\scrD_k, h_z^\eta) - \cL(\scrD_z, h_z^\eta) \\
 & = \E_{(x, y)\sim \scrD_k}[-\log h_z^{\eta}(x, y)] 
 - \E_{(x, y)\sim
  \scrD_z}[-\log h_z^{\eta}(x, y)] \\
  &  = \sum_{x, y} \big( \scrD_z(x, y) - \scrD_k(x, y) \big) \log \left[
  \frac{J_z(x, y)}{K_z(x, y)} \right]\\
&  = \sum_{x, y} \big[ K_z(x, y) - (\scrD_k(x, y) + \eta \, \scrU(x, y) ) \big] \log \left[
  \frac{J_z(x, y)}{K_z(x, y)} \right]\\
  &  = u_k(z) - v_k(z).
\end{align*}
$u_k$ is convex since $-\log J_z$ is convex as the composition of the
convex function $-\log$ with an affine function. Similarly,
$-\log K_z$ is convex, which shows that the second term in the
expression of $v_k$ is a convex function. The first term can be
written in terms of the unnormalized relative entropy:\footnote{The
  unnormalized relative entropy of $P$ and $Q$ is defined by
  $B(P \parallel Q) = \sum_{x, y} P(x, y) \log \left[ \frac {P(x,
      y)}{Q(x, y)} \right] + \sum_{(x, y)} (Q (x, y) - P(x, y))$.}
      \begin{align*}
          &\sum_{x, y} K_z(x, y) \log \left[ \frac {K_z(x, y)}{J_z(x, y)} \right] \\
          &= B(K_z \parallel J_z) + \sum_{(x, y)} (K_z - J_z)(x, y).
      \end{align*}
      The unnormalized relative entropy $B(\cdot \parallel \cdot)$ is
jointly convex \citep{CoverThomas2006},\footnote{To be precise, it can
  be shown that the relative entropy is jointly convex using the
  so-called log-sum inequality \citep{CoverThomas2006}. The same proof
  using the log-sum inequality can be used to show the joint convexity
  of the unnormalized relative entropy.} thus $B(K_z \parallel J_z)$
is convex as the composition of the unnormalized relative entropy with
affine functions (for each of its two arguments).  $(K_z - J_z)$ is an
affine function of $z$ and is therefore convex too.
\end{proof}

\clearpage
\section{Additional experiment results}
\label{app:more_exp}

In this section we provide experiment results on artificial datasets to show that our global objective 
 indeed approaches the known optimal of zero
 with DC-programming algorithm, for both squared loss and cross-entropy loss. We also provide details of our density estimation procedure on the real-world applications, as well as additional experiment results to show that our distribution-weighted predictor \texttt{DW} is robust across various test data mixtures.

\subsection{Artificial dataset}


We first evaluated our algorithm on synthetic datasets, for both
squared loss and cross-entropy loss. 

Consider the following multiple source domain study by \cite{MansourMohriRostamizadeh2009}.  Let $g_1$, $g_2$, $g_3$, $g_4$
denote the Gaussian distributions with means $(1, 1)$, $(-1, 1)$,
$(-1, -1)$, and $(1, -1)$ and unit variance respectively.  Each domain
was generated as a uniform mixture of Gaussians: $\sD_1$ from
$\{g_1, g_2, g_3\}$ and $\sD_2$ from $\{g_2, g_3, g_4\}$.  The
labeling function is $f(x_1, x_2) = x_1^2 + x_2^2$.  We trained linear
regressors for each domain to produce base hypotheses $h_1$ and
$h_2$. Finally, as the true distribution is known for this artificial
example, we directly use the Gaussian mixture
density function to generate our $\sD_k$s.

With this data source, we used our DC-programming solution to find the
optimal mixing weights $z$.
Figure~\ref{fig:synthetic_loss} shows 
the global objective value (of Problem~\ref{eq:opt}) 
vs number of iterations with the uniform
initialization $z_0 = [1/2, 1/2]$. Here, the overall objective
approaches $0.0$, the known global minimum. To verify the robustness
of the solution, we have experimented with various initial conditions
and found that the solution converges to the global solution in each
case. 

\begin{figure}[t]
    \centering
    \includegraphics[height=.3\linewidth]{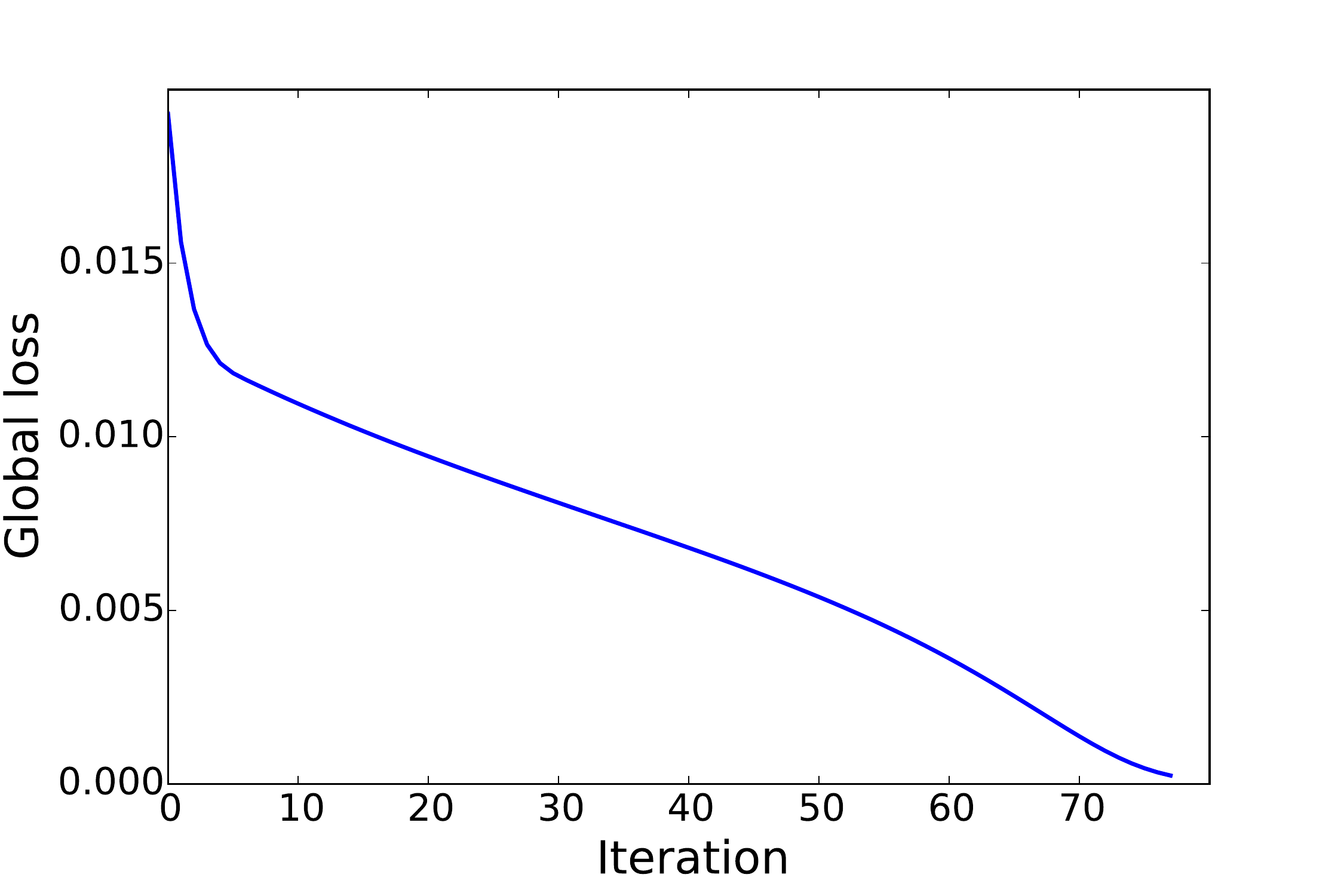}
    \caption{Synthetic global loss versus iteration for squared loss. Our solution converges to the global optimum of zero.}
    \label{fig:synthetic_loss}
\end{figure}


 \begin{figure*}[t]
  	\centering
 	\subfloat[][]{
  	\includegraphics[height=.3\linewidth]{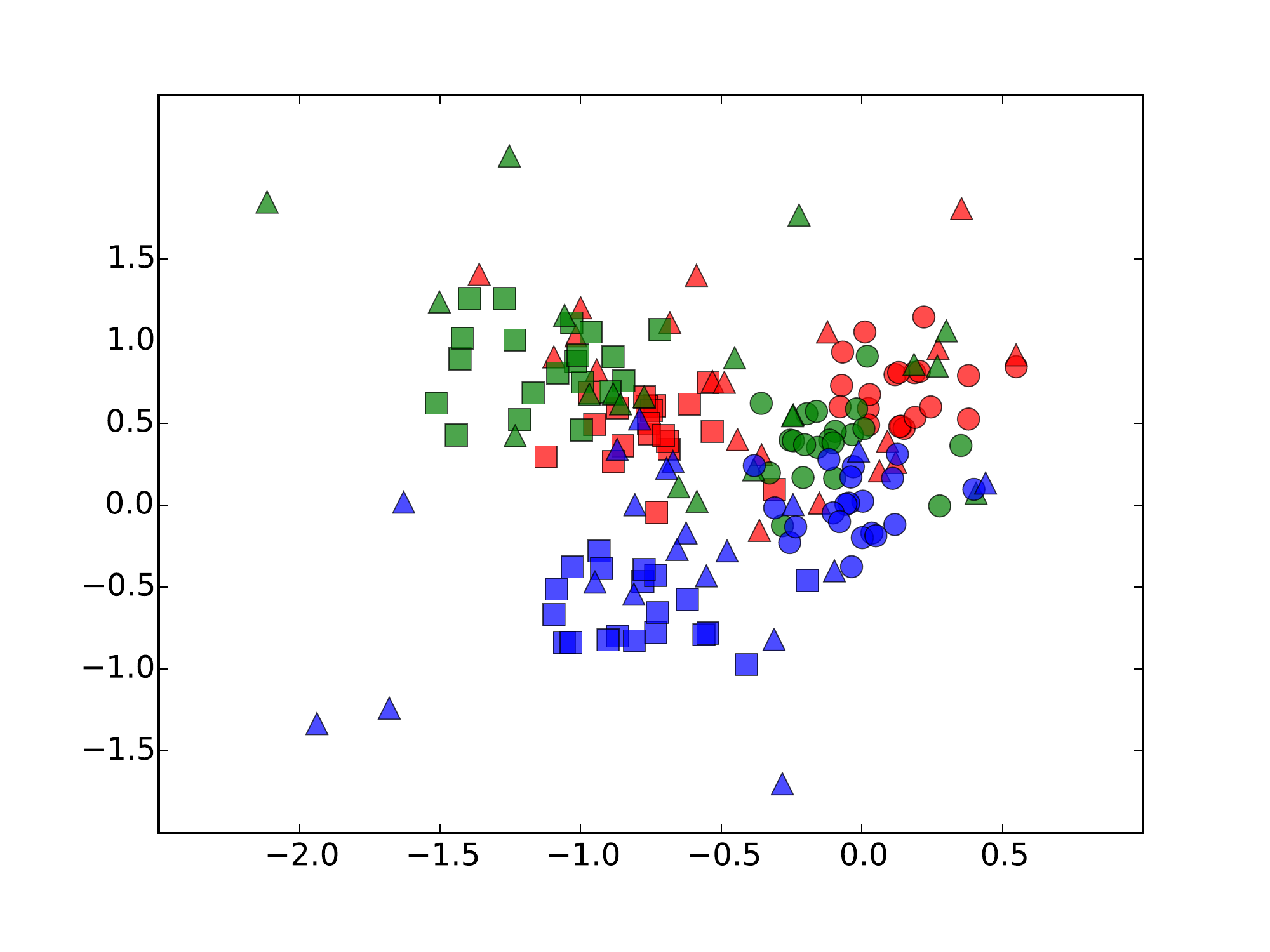}
  	\label{fig:synthetic_pts}
 	}
 	\subfloat[][]{
  	\includegraphics[height=.3\linewidth]{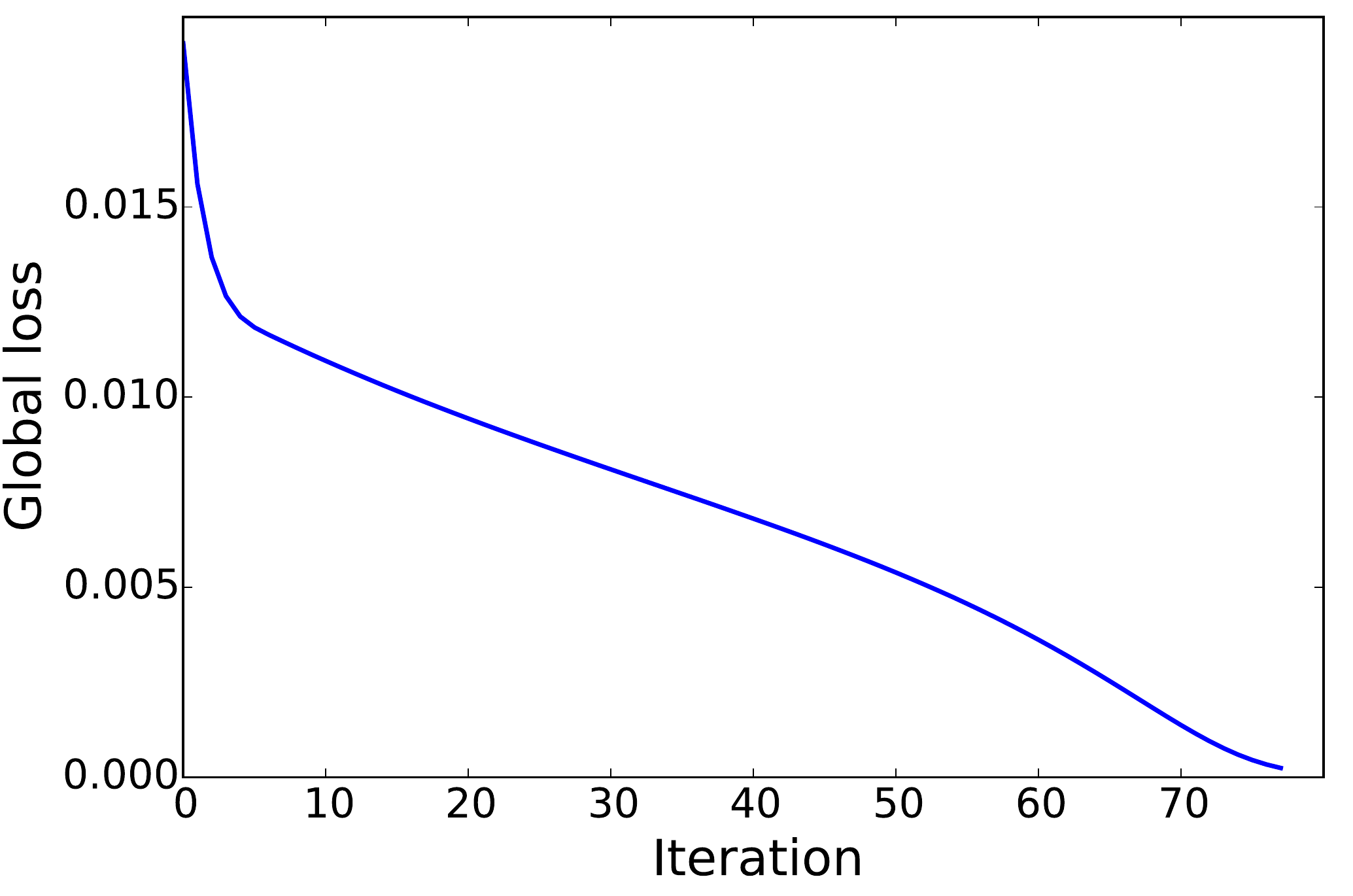}
  	\label{fig:synthetic_loss_logloss}
 	}
 	\caption{\protect\subref{fig:synthetic_pts} Artificial dataset for cross-entropy loss, with three domains (red, green and blue) and three categories (triangle, square, circle). \protect\subref{fig:synthetic_loss_logloss} Artificial dataset global loss versus iteration for cross-entropy loss. We empirically find that our solution converges to the global optimum of zero.}
 \end{figure*}

We next evaluate our algorithm on cross-entropy loss. Here we generate the two-dimensional dataset shown in 
Figure~\ref{fig:synthetic_pts}, which has three domains, denoted in the colors red, green, and blue, 
and three categories, denoted as squares, circles, and triangles. Each domain is generated according 
to a Gaussian mixture model, one mixture per category, with random means. The means of each corresponding 
category across domains are related according to a random fixed orthonormal transformation. 
Finally, the covariance of each mixture is diagonal and fixed across categories. We choose covariance 
magnitudes of 0.05, 0.05, and 0.3 for the red, green, and blue domains, respectively. 
We then train a logistic regression classifier per domain to produce score functions, $h_k$. 
Finally, as the true distribution is known for this artificial example, we forgo density estimation 
and use the Gaussian mixture density function to generate our $\scrD_k$s.

With this data source, we use our DC-programming solution to find the optimal mixing weights, $z$. 
Since only each convex sub-problem is guaranteed to converge, Figure~\ref{fig:synthetic_loss_logloss} reports this 
global loss vs iteration when initializing $z_0 = 1/p$, uniform weights. Here, the overall 
objective approaches 0.0, the known global minimum.
To verify the robustness of the solution, we have experimented with various initial conditions and found the solution 
converges to the global solution from each case.


\begin{table}[t]
\begin{center}
\caption{MSE on sentiment analysis dataset: target domain as various combinations of two domains.}
\scriptsize{
\begin{tabular}{p{0.4cm}P{0.8cm}P{0.8cm}P{0.8cm}P{0.8cm}P{0.8cm}P{0.8cm}}
\toprule
& \multicolumn{6}{c}{Test Data}\\
\cline{2-7}
& \texttt{KD}& \texttt{BE}& \texttt{KB}& \texttt{KE}& \texttt{DB}& \texttt{DE} \\
\midrule
\texttt{K}      &       {1.83$\pm$0.08} &       {1.99$\pm$0.10} &{1.87$\pm$0.08} &       {1.57$\pm$0.06} &       {2.25$\pm$0.08} &       {1.94$\pm$0.10}\\
\texttt{D}      &       {1.95$\pm$0.07} &       {2.11$\pm$0.07} &{2.12$\pm$0.07} &       {2.11$\pm$0.05} &       {1.95$\pm$0.06} &       {1.94$\pm$0.06}\\
\texttt{B}      &       {2.10$\pm$0.09} &       {1.99$\pm$0.08} &{1.96$\pm$0.07} &       {2.21$\pm$0.06} &       {1.87$\pm$0.07} &       {2.13$\pm$0.05}\\
\texttt{E}      &       {2.00$\pm$0.09} &       {1.95$\pm$0.07} &{2.05$\pm$0.05} &       {1.60$\pm$0.05} &       {2.36$\pm$0.07} &       {1.91$\pm$0.07}\\
\texttt{unif}   &       {1.73$\pm$0.06} &       {1.74$\pm$0.07} &{1.74$\pm$0.05} &       {1.62$\pm$0.04} &       {1.85$\pm$0.05} &       {1.73$\pm$0.06}\\
\texttt{KMM}    &       1.83$\pm$0.07   &       1.82$\pm$0.07   &1.78$\pm$0.12   &       1.65$\pm$0.10   &       1.97$\pm$0.13   &       1.88$\pm$0.08\\
\texttt{DW}     &       {\bf1.62$\pm$0.07}      &   {\bf1.61$\pm$0.08}       &      {\bf1.59$\pm$0.05}      &       {\bf1.47$\pm$0.04}      &       {\bf1.75$\pm$0.05}      &       {\bf1.64$\pm$0.05}\\
\bottomrule
\end{tabular}
}
\label{table:sa_more_pairs}	
\end{center}
\end{table}

\begin{table}[t]
\begin{center}
\caption{MSE on the sentiment analysis dataset: target domain as various mixture of four domains:
  $(\textbf{0.4}, 0.2, 0.2, 0.2)$, $(0.2, \textbf{0.4}, 0.2, 0.2)$, 
$(0.2, 0.2, \textbf{0.4}, 0.2)$, $(0.2, 0.2, 0.2, \textbf{0.4})$
of \texttt{K, D, B, E} respectively.}
\scriptsize{
  \begin{tabular}{l cccc}
\toprule
& \multicolumn{4}{c}{Test Data}\\
\cline{2-5}
& \texttt{\textbf{K}DBE}& \texttt{K\textbf{D}BE}& \texttt{KD\textbf{B}E}& \texttt{KDB\textbf{E}} \\
\midrule
\texttt{K}      &       {1.78$\pm$0.05} &       {1.94$\pm$0.10} &       {1.96$\pm$0.08} &       {1.84$\pm$0.07}  \\
\texttt{D}      &       {2.02$\pm$0.10} &       {1.98$\pm$0.10} &       {2.06$\pm$0.11} &       {2.05$\pm$0.09} \\
\texttt{B}      &       {2.01$\pm$0.12} &       {2.01$\pm$0.14} &       {1.94$\pm$0.14} &       {2.06$\pm$0.11} \\
\texttt{E}      &       {1.93$\pm$0.08} &       {2.04$\pm$0.10} &       {2.08$\pm$0.10} &       {1.89$\pm$0.08} \\
\texttt{unif}   &       {1.69$\pm$0.06} &       {1.74$\pm$0.07} &       {1.75$\pm$0.08} &       {1.70$\pm$0.06} \\
\texttt{KMM}    &       1.83$\pm$0.12   &       1.92$\pm$0.14   &       1.87$\pm$0.15   &       1.85$\pm$0.13	\\
\texttt{DW}     &       {\bf1.55$\pm$0.08}      &       {\bf1.62$\pm$0.08}      &       {\bf1.59$\pm$0.09}      &       {\bf1.56$\pm$0.08}      \\
\bottomrule
\end{tabular}
}
\label{table:sa_lamset1}	
\end{center}
\end{table}

\begin{table}[t]
\begin{center}
\caption{MSE on the sentiment analysis dataset: target domain as various mixture of four domains:
  $(\textbf{0.3}, \textbf{0.3}, 0.2, 0.2)$, $(\textbf{0.3}, 0.2, \textbf{0.3}, 0.2)$, 
  $(\textbf{0.3}, 0.2, 0.2, \textbf{0.3})$, $(0.2, \textbf{0.3}, \textbf{0.3}, 0.2)$,
  $(0.2, \textbf{0.3}, 0.2, \textbf{0.3})$, $(0.2, 0.2, \textbf{0.3}, \textbf{0.3})$
of \texttt{K, D, B, E} respectively.}
\scriptsize{
\begin{tabular}{p{0.3cm}P{0.8cm}P{0.8cm}P{0.8cm}P{0.8cm}P{0.8cm}P{0.8cm}}
\toprule
& \multicolumn{6}{c}{Test Data}\\
\cline{2-7}
& \texttt{\textbf{KD}BE}& \texttt{\textbf{K}D\textbf{B}E}& \texttt{\textbf{K}DB\textbf{E}}& \texttt{K\textbf{DB}E} & \texttt{K\textbf{D}B\textbf{E}} & \texttt{KD\textbf{BE}}\\
\midrule
\texttt{K}      &       {1.86$\pm$0.10} &       {1.87$\pm$0.07} &       {1.79$\pm$0.08}      &       {1.96$\pm$0.10} &       {1.89$\pm$0.10} &       {1.89$\pm$0.08}\\
\texttt{D}      &       {2.01$\pm$0.13} &       {2.05$\pm$0.12} &       {2.04$\pm$0.12}      &       {2.03$\pm$0.12} &       {2.02$\pm$0.13} &       {2.06$\pm$0.12}\\
\texttt{B}      &       {2.01$\pm$0.15} &       {1.98$\pm$0.14} &       {2.05$\pm$0.13}      &       {1.98$\pm$0.15} &       {2.04$\pm$0.14} &       {2.01$\pm$0.13}\\
\texttt{E}      &       {2.00$\pm$0.10} &       {2.01$\pm$0.09} &       {1.91$\pm$0.08}      &       {2.08$\pm$0.10} &       {1.97$\pm$0.08} &       {1.99$\pm$0.08}\\
\texttt{unif}   &       {1.72$\pm$0.09} &       {1.72$\pm$0.08} &       {1.69$\pm$0.07}      &       {1.75$\pm$0.08} &       {1.72$\pm$0.08} &       {1.73$\pm$0.08}\\
\texttt{KMM}    &       1.85$\pm$0.16   &       1.86$\pm$0.14   &       1.85$\pm$0.15   &       1.90$\pm$0.14   &       1.89$\pm$0.16   &       1.90$\pm$0.14\\
\texttt{DW}     &       {\bf1.58$\pm$0.10}      &   {\bf1.57$\pm$0.10}       &       {\bf1.55$\pm$0.09}      &       {\bf1.61$\pm$0.10}      &       {\bf1.59$\pm$0.08}      &       {\bf1.58$\pm$0.09}\\
\bottomrule
\end{tabular}
}
\label{table:sa_lamset2}	
\end{center}
\end{table}

\subsection{Sentiment analysis task for squared loss}
We begin by detailing our density estimation method for the sentiment analysis experiment. We first used the same vocabulary
 defined for feature extraction to train a separate bigram statistical
 language model for each domain, using the OpenGrm library
 \citep{roark2012opengrm}.  Next, we randomly draw a sample set $S_k$
 of $10\mathord,000$ sentences from each bigram language model. We
 define $\h \sD_k$ to be the empirical distribution of $S_k$, which is
 a very close estimate of marginal distribution of the language model,
 thus it is also a good estimate of $\sD_k$. We approximate the label
 of a randomly generated sample $x_i$ by taking the average of the
 $h_k$s:
 $y_i = \sum_{\{k\colon x_i \in S_k\}} h_k(x_i) / \vert \{k\colon x_i
 \in S_k\}\vert$.  These randomly drawn samples were used to find the
 fixed-point $z$.

 Note that we only use estimates of the marginal distributions
 (language models) to find $z$ and do not use any labels.  We use the original product review text and rating labels for testing. Their
 densities $\h \sD_k$ were estimated by the bigram language models
 directly, therefore a close estimate of $\sD_k$.
 
Next we compare \texttt{DW} to accessible predictors on
various test mixture domains.
Table~\ref{table:sa_more_pairs} shows MSE on all combinations of two domains.
Table~\ref{table:sa_lamset1},~\ref{table:sa_lamset2} reports MSE on additional test  mixture domains. 
The first four target mixtures correspond to various orderings of $(0.4, 0.2, 0.2, 0.2)$. 
The next six target mixtures correspond to various orderings of $(0.3, 0.3, 0.2, 0.2)$. 
In column titles we bold the domain(s) with highest weight.

In all these experiments, our distribution-weighted predictor \texttt{DW}
outperforms all competing baselines: the source only baselines for each domain, 
\texttt{K, D, B, E}, a uniform weighted predictor \texttt{unif}, and \texttt{KMM}.

\subsection{Recognition tasks for cross-entropy loss}
Here, we describe our density estimation technique for the object recognition task.

To estimate the per domain densities, we first extract per image features using the in-domain ConvNet model,
and then estimate the marginal distribution $\msD_k(x)$ over the per domain collection of features,
using non-parametric kernel density estimation
with a Gaussian kernel and a cross-validated bandwidth parameter.
We use estimated marginals $\h \msD_k$ instead of estimated joint distributions $\h \sD_k$,
because when the conditional probabilities are the same across domains and when
$\eta\to 0$, $ h_z^\eta(x,y)$ converges to a normalized predictor
$ \widetilde h_z(x,y) = \sum_{k = 1}^p\frac{z_k \msD_k(x) }{
\sum_{j = 1}^p z_j \msD_j(x) } h_k(x, y)$. 
Thus in our experiments, we approximate $\h h_z^\eta(x,y)$ with $\widetilde{ \h h_z}(x,y)$ using our estimated marginal distributions $\h\msD_k(x)$.

\clearpage
\section{R\'enyi Divergence}
\label{app:renyi}
The R\'enyi Divergence measures the divergence between two 
distributions. 
The R\'enyi Divergence
is parameterized by $\alpha$ and denoted by $\sfD_\alpha$. The
$\alpha$-R\'enyi Divergence of two distributions $\sD$ and $\sD'$ is
defined by
\begin{equation}
  \sfD_\alpha(\sD \parallel \sD') = \frac{1}{\alpha - 1}
  \log \hspace{-0.3cm} \sum_{(x, y) \in \cX \times \cY} \hspace{-0.3cm}
  \sD(x, y)  \left[ \frac{\sD(x, y)}{\sD'(x, y)} \right]^{\alpha - 1} .
\end{equation}
It can be shown that the R\'enyi Divergence is always non-negative and
that for any $\alpha > 0$, $\sfD_\alpha(\sD \parallel \sD') = 0$ iff
$\sD = \sD'$, (see \citep{arndt}).  
We will denote by $\sfd_\alpha(\sD \parallel \sD')$ the exponential:
\begin{equation}
  \sfd_\alpha(\sD \parallel \sD') = e^{\sfD_\alpha(\sD \parallel \sD')} 
  =\Bigg[ \sum_{(x, y) \in \cX \times \cY} \frac{\sD^\alpha(x, y)}{\sD'^{\alpha - 1}(x, y)}
  \Bigg]^{\frac{1}{\alpha - 1}}. 
\end{equation}

R\'enyi divergence 
(and $\sfd_\alpha(\sD \parallel \sD')$) is 
nondecreasing as a function of $\alpha$, and 
\begin{equation}
    \sfd_\alpha(\sD \parallel \sD') \le \sfd_\infty(\sD \parallel \sD') = 
    \sup_{(x,y)\in\cX\times\cY} \left[\frac{\sD(x,y)}{\sD'(x,y)}\right].
\end{equation}

\end{document}